\documentclass[11pt]{article}

\usepackage{amssymb,amsmath,amsthm,bbm}
\usepackage{verbatim,float,url,dsfont}
\usepackage{graphicx,subcaption,psfrag}
\usepackage{algorithm}
\usepackage{algpseudocode}
\usepackage{mathtools,enumitem}
\usepackage{multirow}
\usepackage{ragged2e}
\usepackage{xr-hyper}
\usepackage{caption}
\usepackage{array}

\usepackage[utf8]{inputenc} 
\usepackage[T1]{fontenc}    
\usepackage{booktabs}       
\usepackage{nicefrac}         
\usepackage{microtype}      

\usepackage[dvipsnames]{xcolor} 
\definecolor{hunyuanblue}{HTML}{1E4A8F}
\usepackage{pifont} 
\usepackage{etoc} 
\usepackage{tikz} 
\usepackage{pgfplots}  
\pgfplotsset{compat=1.16}  

\ifdefined\TimesFont 
\usepackage{times} 
\fi

\ifdefined\ParSkip 
\usepackage{parskip} 
\fi

\newtheorem*{assumption*}{\assumptionnumber}
\providecommand{\assumptionnumber}{}


\makeatletter
\newcommand*\rel@kern[1]{\kern#1\dimexpr\macc@kerna}
\newcommand*\widebar[1]{%
  \begingroup
  \def\mathaccent##1##2{%
    \rel@kern{0.8}%
    \overline{\rel@kern{-0.8}\macc@nucleus\rel@kern{0.2}}%
    \rel@kern{-0.2}%
  }%
  \macc@depth\@ne
  \let\math@bgroup\@empty \let\math@egroup\macc@set@skewchar
  \mathsurround\z@ \frozen@everymath{\mathgroup\macc@group\relax}%
  \macc@set@skewchar\relax
  \let\mathaccentV\macc@nested@a
  \macc@nested@a\relax111{#1}%
  \endgroup
}
\makeatother

\DeclareMathOperator{\Var}{Var}

\def\E{\mathbb{E}}

\def\R{\mathbb{R}}

\usepackage{amsmath,amsfonts,bm,amsthm}

\newcommand{\epsilonv}{\bm{\epsilon}}

\def\1{\bm{1}}

\def\rmI{{\mathbf{I}}}

\def\vmu{{\bm{\mu}}}

\def\va{{\bm{a}}}

\def\vc{{\bm{c}}}
\def\vd{{\bm{d}}}

\def\vg{{\bm{g}}}

\def\vs{{\bm{s}}}

\def\vv{{\bm{v}}}
\def\vw{{\bm{w}}}
\def\vx{{\bm{x}}}

\DeclareMathAlphabet{\mathsfit}{\encodingdefault}{\sfdefault}{m}{sl}
\SetMathAlphabet{\mathsfit}{bold}{\encodingdefault}{\sfdefault}{bx}{n}

\def\gC{{\mathcal{C}}}

\def\gL{{\mathcal{L}}}

\def\gN{{\mathcal{N}}}

\def\gV{{\mathcal{V}}}

\newcommand{\KL}{{D}_{\mathrm{KL}}}
\newcommand{\TV}{{D}_{\mathrm{TV}}}

\usepackage{hunyuan}

\usepackage[capitalize,noabbrev]{cleveref}
\usepackage{tabularx}
\usepackage[most]{tcolorbox}
\usepackage{wrapfig}
\usepackage{placeins}   
\usepackage{xspace}
\usepackage{fancyhdr}
\usepackage{xurl}
\usepackage{colortbl}   
\usepackage{titletoc}   

\makeatletter
\long\def\@makecaption#1#2{%
  \vskip 10pt
  \setbox\@tempboxa\hbox{#1: #2}%
  \ifdim \wd\@tempboxa >\hsize
    \noindent #1: #2\par   
  \else
    \hbox to\hsize{\hfil\box\@tempboxa\hfil}
  \fi}
\makeatother

\hypersetup{
  colorlinks=true,
  linkcolor=hunyuanblue,
  citecolor=hunyuanblue,
  urlcolor=hunyuanblue,
}

\makeatletter
\def\section{\@startsiction{section}{1}{\z@}{-0.24in}{0.10in}
             {\large\bf\raggedright\color{hunyuanblue}}}
\def\subsection{\@startsection{subsection}{2}{\z@}{-0.20in}{0.08in}
                {\normalsize\bf\raggedright\color{hunyuanblue}}}
\makeatother

\fancypagestyle{plain}{%
  \fancyhf{}%
  \fancyfoot[C]{\thepage}%
}

\fancypagestyle{firststyle}{%
  \fancyhf{}%
  \fancyhead[L]{\raisebox{-18.5mm}[0pt][0pt]{\includegraphics[height=12mm]{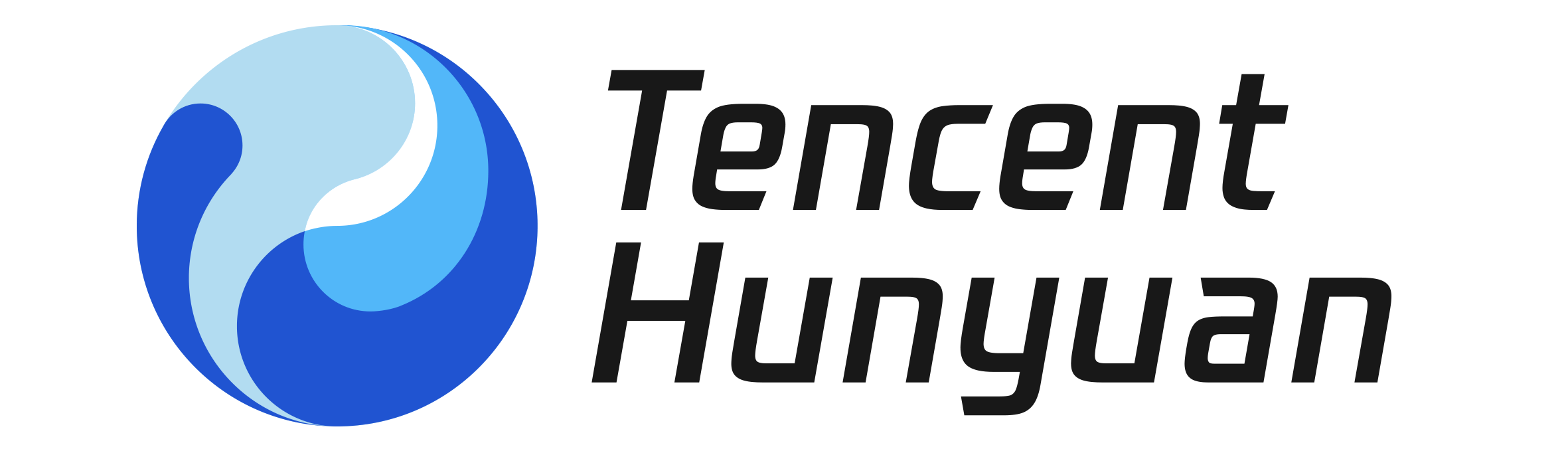}}}%
  \fancyfoot[C]{\thepage}%
}

\makeatletter
\newcommand{\appendixtitle}{%
  \clearpage
  \thispagestyle{plain}%
  \vbox{%
    \hsize\textwidth
    \linewidth\hsize
    \vskip 0.1in
    {\color{hunyuanblue}\hrule height 0.6pt}%
    \vskip 4mm
    \centering
    {\LARGE\bfseries Appendix of \@title\par}%
    \vskip 4mm
    {\color{hunyuanblue}\hrule height 0.6pt}%
    \vskip 0.3in\@minus0.1in
  }%
}
\makeatother

\newcommand{\answerTODO}[1][]{\textcolor{red}{\bf [TODO]}}

\newcolumntype{Y}{>{\raggedright\arraybackslash}X}

\definecolor{caseblue}{RGB}{42, 91, 160}
\definecolor{casebg}{RGB}{246, 248, 252}
\definecolor{caseborder}{RGB}{180, 195, 220}
\definecolor{paperblue}{HTML}{1F77B4}
\definecolor{paperred}{HTML}{D62728}
\definecolor{deepred}{HTML}{B22222}
\definecolor{softred}{HTML}{C44E52}

\newtcolorbox[auto counter, number within=section]{compactcase}[2][]{
  breakable,
  enhanced,
  colback=gray!2,
  colframe=gray!30,
  colbacktitle=gray!12,
  coltitle=black,
  fonttitle=\bfseries,
  title={#2},
  boxrule=0.45pt,
  arc=1mm,
  left=1.5mm,
  right=1.5mm,
  top=1mm,
  bottom=1mm,
  toptitle=0.7mm,
  bottomtitle=0.7mm,
  before skip=0.8em,
  after skip=0.8em,
  label={#1}
}

\definecolor{abstractbg}{HTML}{F0F7FC}

\title{\texttt{Flow-DPPO}: Divergence Proximal Policy Optimization for Flow Matching Models}

\begin{document}

\thispagestyle{firststyle}
\vspace*{0.25cm}
{\color{hunyuanblue}\hrule height 0.6pt}
\vskip 6mm
\begin{center}
{\LARGE\bfseries \texttt{Flow-DPPO}: Divergence Proximal Policy Optimization\\[2pt]
for Flow Matching Models\par}
\end{center}
\vskip 3mm
{\color{hunyuanblue}\hrule height 0.6pt}
\vskip 6mm
\begin{center}
\textbf{Bowen Ping}$^{1,2,*}$ \quad
\textbf{Xiangxin Zhou}$^{2,*\,\mathparagraph}$ \quad
\textbf{Penghui Qi}$^{3}$ \\[4pt]
\textbf{Minnan Luo}$^{1,\ddagger}$ \quad
\textbf{Liefeng Bo}$^{2}$ \quad
\textbf{Tianyu Pang}$^{2,\ddagger}$
\\[8pt]
$^1$Xi'an Jiaotong University \quad
$^2$Tencent Hunyuan \quad
$^3$National University of Singapore \\[6pt]
{\small $^*$Equal contribution \quad
$^\mathparagraph$Project Lead \quad
$^\ddagger$Corresponding author}
\end{center}
\vskip 6mm

\begin{tcolorbox}[
  colframe=abstractbg,
  colback=abstractbg,
  boxrule=0pt,
  arc=2mm,
  enhanced,
  top=12pt,
  bottom=12pt,
  left=15pt,
  right=15pt,
  width=\textwidth,
]
\textbf{Abstract.}\quad
Recent work has demonstrated that online reinforcement learning (RL) can substantially improve the quality and alignment of flow matching models for image and video generation. Methods such as Flow-GRPO and CPS cast the denoising process as a Markov Decision Process and apply PPO-style ratio clipping to enforce a trust region. However, we argue that \emph{ratio clipping is structurally ill-suited for flow models}: the probability ratio between new and old policies is a noisy, single-sample estimate of the true policy divergence, leading to over-constraining in some regions of the trajectory and under-constraining in others.
We propose \textbf{Flow-DPPO} (Flow Divergence Proximal Policy Optimization), which replaces ratio clipping with a divergence proximal constraint. A key observation is that the per-step policy in flow models is Gaussian, enabling \emph{exact} and \emph{cheap} computation of the KL divergence between old and new policies.
Flow-DPPO employs an asymmetric divergence mask that blocks gradient updates only when they simultaneously move away from the trusted region and violate the divergence threshold.
Experiments show that Flow-DPPO achieves higher rewards with better KL-proximal efficiency, alleviates catastrophic forgetting, promotes balanced multi-objective optimization, and enables stable multi-epoch training where ratio clipping degrades.

\vskip 8pt
\textbf{Date:} June 8, 2026 \\
\textbf{Code:} \url{https://github.com/Tencent-Hunyuan/UniRL/tree/main/FlowDPPO}
\end{tcolorbox}

\section{Introduction}
\label{sec:introduction}

Reinforcement learning (RL) has emerged as a core paradigm for aligning models with downstream objectives. In language models, RL methods such as DPO \citep{rafailov2023direct} and GRPO \citep{grpo} have substantially improved alignment \citep{ouyang2022training} and reasoning capabilities \citep{guo2025deepseekr1}. Recently, these advances have been extended to image and video generation~\cite{liu2026flowgrpo,wallace2024diffusion,wang2025cps,awm,diffusionnft}, where
flow matching models~\citep{lipman2023flow, liu2023flow} represent the dominant generative framework. Among them, Flow-GRPO~\citep{liu2026flowgrpo} and DanceGRPO~\citep{xue2025dancegrpo} demonstrated strong performance by transforming deterministic ODE sampling into stochastic SDE trajectories and introducing PPO-style ratio clipping to enforce \emph{trust-region optimization}.

The theoretical foundation of trust-region methods originates from Trust Region Policy Optimization (TRPO)~\citep{schulman2015trust}, which establishes a \emph{policy improvement bound}: monotonic improvement is guaranteed when policy updates remain within a trust region defined by the divergence between the old and new policies. PPO~\citep{schulman2017proximal} later introduced ratio clipping as a computationally efficient first-order approximation to TRPO. However, as noted by \citet{qi2026rethinking}, each clipping decision is based on a single-sample Monte Carlo estimate of the true Total Variation (TV) divergence, rather than the divergence itself. In the continuous and high-dimensional latent space of flow models, this estimation noise becomes substantially amplified, leading to a systematic left shift in the ratio distribution, with its mean falling below one~\citep{grpoguard}. We show that this bias is intrinsic to Gaussian policies: the standard PPO clipping range $[1{-}\epsilon,1{+}\epsilon]$ therefore becomes effectively asymmetric, failing to adequately constrain over-optimization for positive-advantage samples while excessively clipping negative-advantage ones.

\begin{figure*}[t]
    \centering
    \setlength{\tabcolsep}{0pt}
    \renewcommand{\arraystretch}{0}
    \newcommand{\TeaserImg}[1]{\includegraphics[width=0.166\textwidth,height=0.166\textwidth]{figs/case_comparison/images/#1}}
    \newcommand{\TeaserPrompt}[1]{\parbox[c]{0.15\textwidth}{\centering\fontsize{7}{6}\selectfont\textit{#1}}}
    \begin{tabular}{>{\centering\arraybackslash}m{0.15\textwidth}*{5}{>{\centering\arraybackslash}m{0.166\textwidth}}}
      & \multicolumn{1}{c}{\scriptsize FLUX.1-dev} & \multicolumn{1}{c}{\scriptsize Flow-GRPO} & \multicolumn{1}{c}{\scriptsize Flow-CPS} & \multicolumn{1}{c}{\scriptsize GRPO-Guard} & \multicolumn{1}{c}{\scriptsize \textbf{Flow-DPPO}} \\[2pt]
      \TeaserPrompt{seven green croissants}
      & \TeaserImg{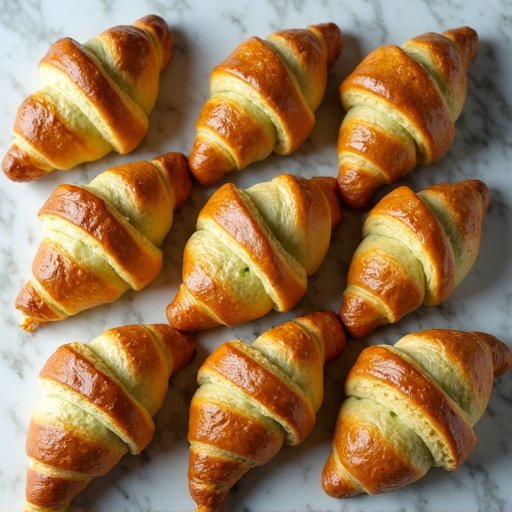}
      & \TeaserImg{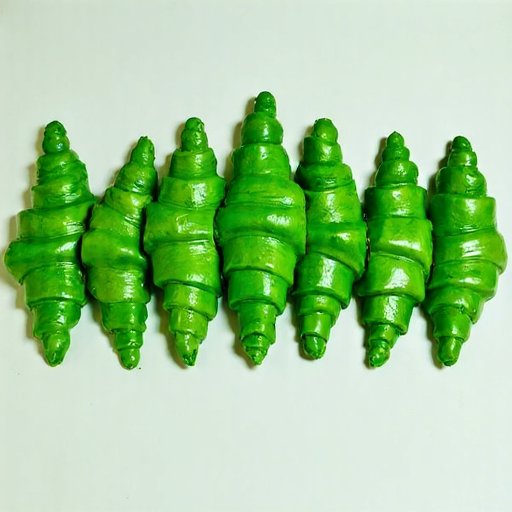}
      & \TeaserImg{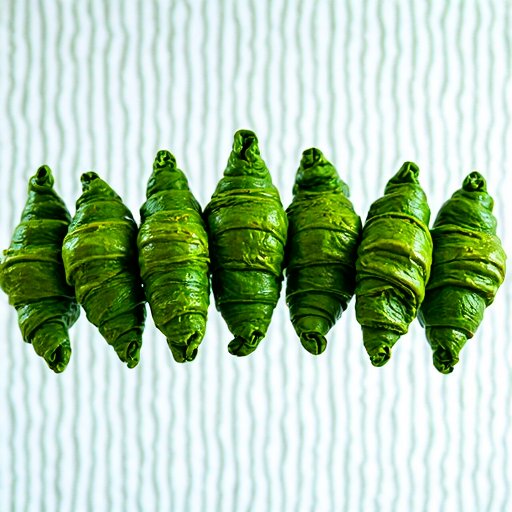}
      & \TeaserImg{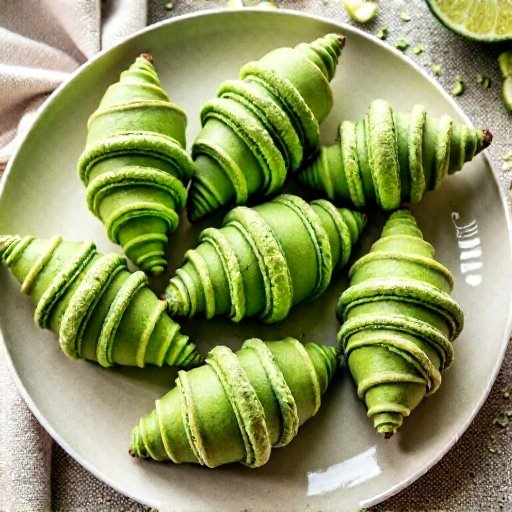}
      & \TeaserImg{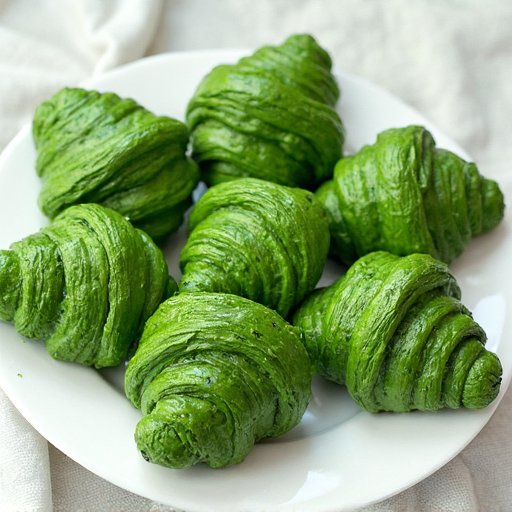} \\
      \TeaserPrompt{a blue dog on top of three white sheep behind seven white candles}
      & \TeaserImg{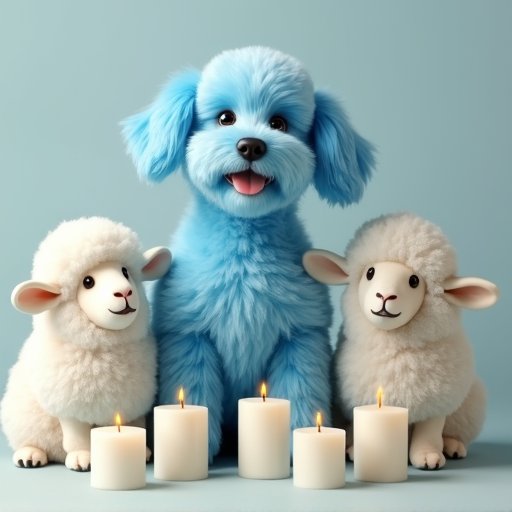}
      & \TeaserImg{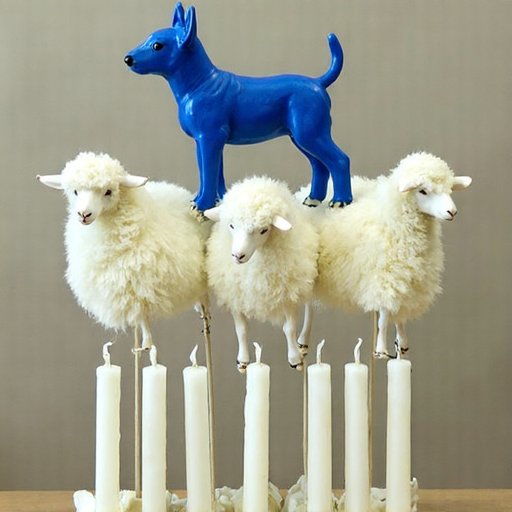}
      & \TeaserImg{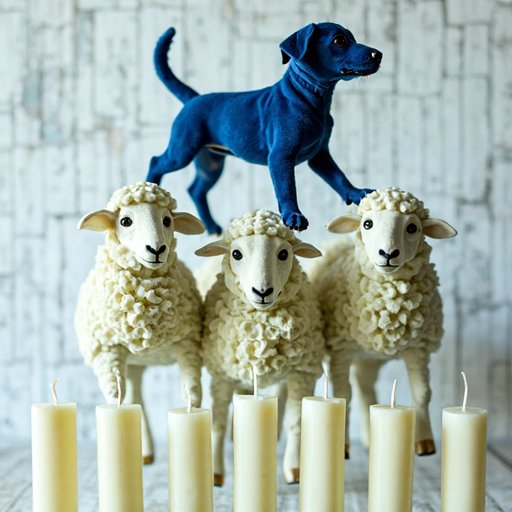}
      & \TeaserImg{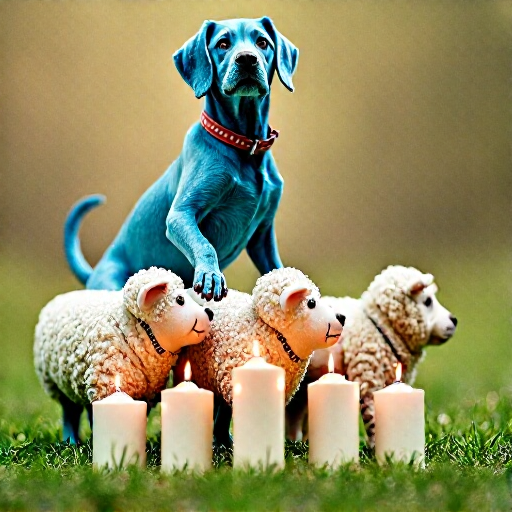}
      & \TeaserImg{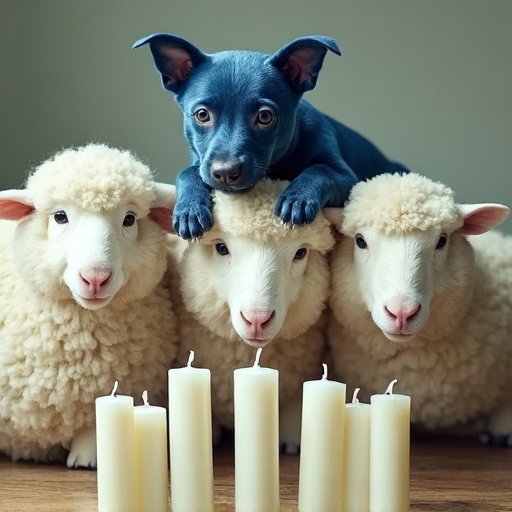} \\
      \TeaserPrompt{a blue giraffe behind seven pink clocks to the right of an elephant}
      & \TeaserImg{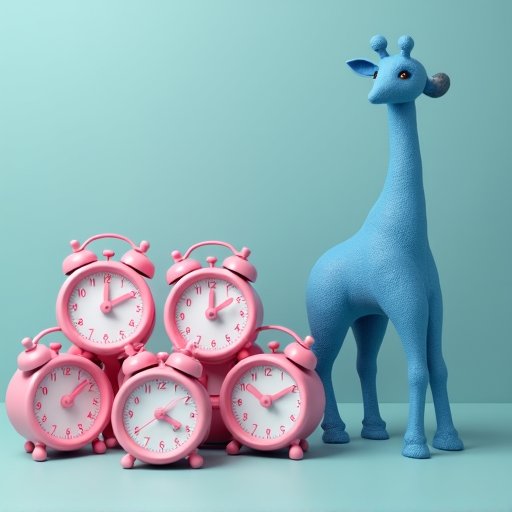}
      & \TeaserImg{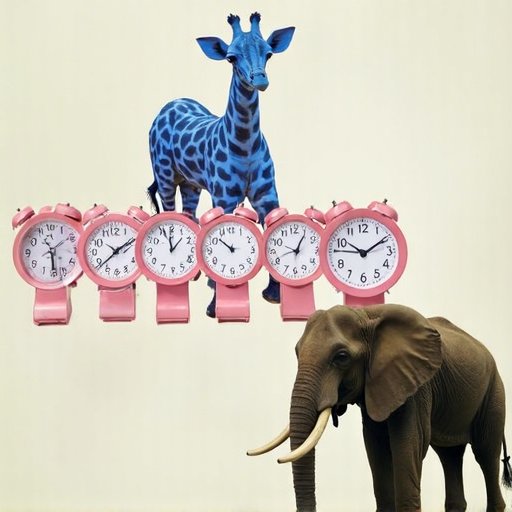}
      & \TeaserImg{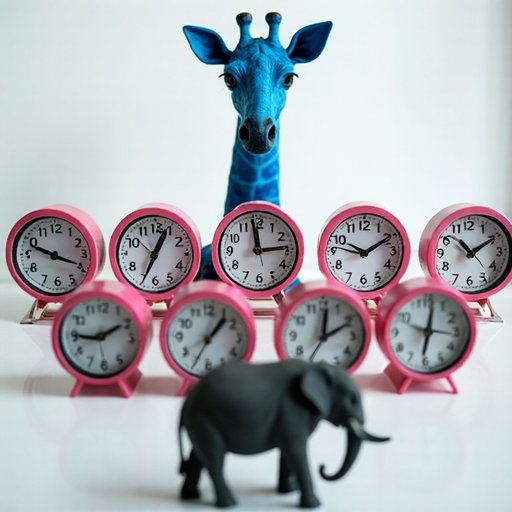}
      & \TeaserImg{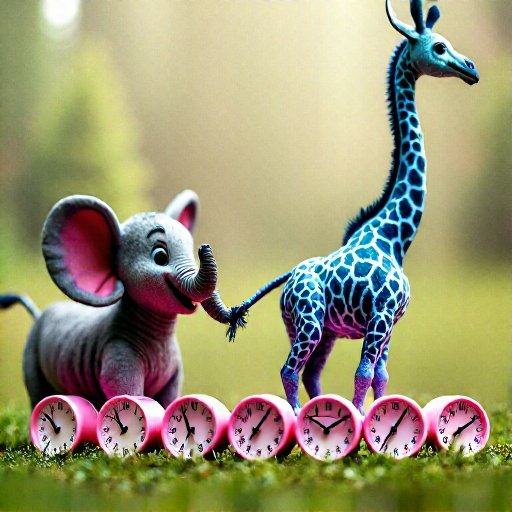}
      & \TeaserImg{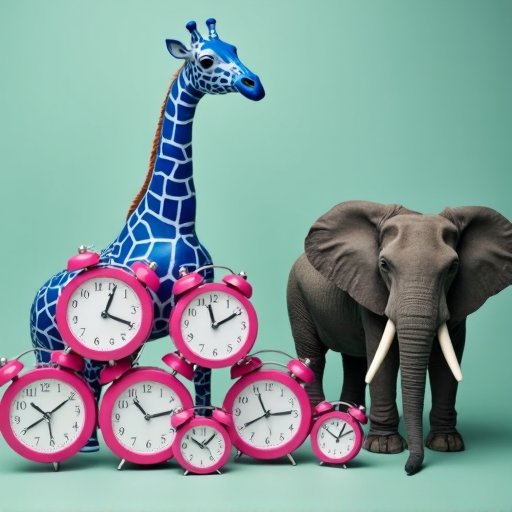}
    \end{tabular}
    \vspace{-0.5em}
    \caption{Qualitative comparison on FLUX.1-dev~\citep{flux1dev} with GenEval2~\citep{geneval2} prompts. Flow-DPPO achieves competitive compositional accuracy with notably less image quality degradation compared to Flow-GRPO~\citep{liu2026flowgrpo}, Flow-CPS~\citep{wang2025cps}, and GRPO-Guard~\citep{grpoguard}, reflecting their superior KL-proximal efficiency.}
    \label{fig:teaser}
    \vspace{-.7em}
  \end{figure*}

To mitigate this bias, GRPO-Guard~\citep{grpoguard} proposed normalizing the ratio distribution.
While this re-centering alleviates the symptom, it does not address the root cause: the ratio remains a noisy, per-sample proxy for the true policy divergence.
We observe that flow models offer a structural advantage that sidesteps this problem entirely.
Because each per-step policy is \textit{Gaussian} with a mean $\vmu_\theta$ determined by the velocity network and a fixed, schedule-dependent variance $\sigma$,
the KL divergence between old and new policies reduces to $\|\vmu_{\theta_\mathrm{old}} - \vmu_\theta\|^2 / (2\sigma^2)$,
which is an \emph{exact}, \emph{deterministic} quantity that can be computed from two forward passes already performed during training.
Unlike the LLM setting, where DPPO~\citep{qi2026rethinking} must resort to approximate divergence reductions over large vocabularies, flow models admit exact divergence computation at no additional cost.
This motivates replacing ratio clipping with a direct KL-proximal trust region constraint.

Building on this insight, we propose \textbf{Flow-DPPO} (Flow \underline{D}ivergence \underline{P}roximal \underline{P}olicy \underline{O}ptimization), which replaces ratio clipping with a divergence-based mask.
The mask blocks gradient updates only when two conditions are jointly met:
(1) the advantage and ratio indicate that the update is moving the policy \emph{away} from the old policy,
and (2) the exact KL divergence already exceeds a threshold.
This design directly enforces the trust region while preserving the beneficial asymmetric structure of PPO:
updates that move the policy \emph{towards} the old policy are never blocked, accelerating recovery from overshooting.
Extensive experiments on various base models  demonstrate that Flow-DPPO achieves
superior reward optimization,
improved KL-proximal efficiency,
stronger robustness to catastrophic forgetting,
balanced multi-objective optimization that mitigates reward hacking,
and stable multi-epoch training that enables higher sample efficiency.
Figure~\ref{fig:teaser} presents qualitative generation results demonstrating that Flow-DPPO achieves competitive compositional accuracy while preserving notably higher visual quality than existing methods.

\vspace{-0.15cm}
\section{Preliminaries}
\label{sec:background}
\vspace{-0.15cm}

Flow matching~\citep{lipman2023flow,liu2023flow} learns a continuous-time velocity field that transports samples from a simple source distribution to the data distribution. Specifically, let $\vx_0 \sim \pi_0 = p_{\text{data}}$, and define an interpolating path $\vx_t = \alpha_t \vx_0 + \sigma_t \epsilonv$ with $\epsilonv \sim \gN(\bm{0}, \rmI)$,
where $\alpha_t$ and $\sigma_t$ determine the probability path between data and noise. This construction induces a conditional distribution $\pi_{t|0}(\vx_t \mid \vx_0) = \gN(\alpha_t \vx_0, \sigma_t^2 \rmI)$. The goal of flow matching is to train a time-dependent vector field $\vv_\theta(\vx_t, t)$ to match the target velocity, which is given by $\vv = \frac{\mathrm{d}\vx_t}{\mathrm{d}t}
    = \dot{\alpha}_t \vx_0 + \dot{\sigma}_t \epsilonv$ and the functinal $\dot{f}_t \coloneq \mathrm{d}f_t / \mathrm{d}t$. The model $\vv_\theta(\vx_t, t)$ is then trained by minimizing the regression objective
\begin{equation}
\label{eq:diffusion_loss}
    \E_{t,\,\vx_0 \sim \pi_0,\,\epsilonv \sim \gN(\bm{0}, \rmI)}
    \bigl[w(t)\|\vv_\theta(\vx_t,t)-\vv\|_2^2\bigr]\textrm{,}
\end{equation}
where $w(t)$ is a weighting function.
After training, samples are generated by solving the ODE $\frac{\mathrm{d}\vx_t}{\mathrm{d}t} = \vv_\theta(\vx_t, t)$.
In practice, simple numerical solvers such as Euler discretization are often sufficient for high-quality sampling~\citep{karras2022elucidating,lu2022dpm,song2021denoising}.
A notable special case is rectified flow~\citep{liu2023flow}, which uses the linear conditional path $\alpha_t = 1-t$ and $\sigma_t = t$. Under this choice, the target velocity reduces to
$\vv = \epsilonv - \vx_0$. We adopt this linear schedule throughout the paper.

\subsection{RL Fine-Tuning for Flow Matching Models}
\label{subsec:RL_finetuning_for_flow_matching_models}

For text-conditional flow matching models, given a conditioning prompt $\vc$, generation starts from a Gaussian latent $\vx_T \sim \gN(\bm{0}, \rmI)$ and progressively transforms it into a clean sample $\vx_0$. At each timestep $t$, the flow model predicts a velocity field $\vv_\theta(\vx_t, t, \vc)$,
which specifies a deterministic generation direction. 
Applying RL algorithms such as GRPO \citep{grpo} to flow matching models requires a sampler-induced stochastic policy at each denoising step. Flow-GRPO \citep{liu2026flowgrpo} constructs such a policy via an ODE-to-SDE conversion, which transforms the probability-flow ODE into an equivalent SDE with the same marginals \citep{albergo2023building,albergo2024stochastic,song2021scorebased}:
$\mathrm{d}\vx_t
=
\left[
\vv_\theta(\vx_t, t, \vc)
+
\frac{\sigma_t^2}{2t}
\big(\vx_t + (1-t)\vv_\theta(\vx_t, t, \vc)\big)
\right]\mathrm{d}t
+
\sigma_t\,\mathrm{d}\vw$,
where $\mathrm{d}\vw$ denotes Wiener process increments, $\sigma_t = a\sqrt{\frac{t}{1-t}}$, and $a$ is a scalar hyperparameter controlling the noise level. Applying Euler--Maruyama discretization yields the Flow-SDE sampler:
\begin{equation}
\vx_{t-\Delta t}
=
\vx_t
+
\left[
\vv_\theta(\vx_t, t, \vc)
+
\frac{\sigma_t^2}{2t}
\big(\vx_t + (1-t)\vv_\theta(\vx_t, t, \vc)\big)
\right]\Delta t
+
\sigma_t\sqrt{\Delta t}\,\epsilonv\textrm{,}
\label{eq:flow_sde}
\end{equation}
with $\epsilonv \sim \gN(\bm{0}, \rmI)$.
An alternative is Coefficients-Preserving Sampling (CPS) \citep{wang2025cps}, which reduces the excessive noise injection in Flow-SDE and better preserves the interpolation structure of the scheduler. Let
\(
\hat{\vx}_0 = \vx_t - t\,\hat{\vv}_\theta(\vx_t, t, \vc),
\hat{\vx}_1 = \vx_t + (1-t)\,\hat{\vv}_\theta(\vx_t, t, \vc)
\)
denote the predicted clean sample and noise component, respectively. CPS updates the latent as
\begin{equation}
\label{eq:cps}
\vx_{t-\Delta t}
=
(1-(t-\Delta t))\,\hat{\vx}_0
+
(t-\Delta t)\cos\!\left(\frac{\eta\pi}{2}\right)\hat{\vx}_1
+
(t-\Delta t)\sin\!\left(\frac{\eta\pi}{2}\right)\epsilonv\textrm{,}
\end{equation}
where $\epsilonv \sim \gN(\bm{0}, \rmI)$ and $\eta\in[0,1]$ controls the stochasticity.
Both Flow-SDE and CPS therefore induce Gaussian per-step policies written as
\begin{equation}
\label{eq:abstract_policy}
p_\theta(\vx_{t-\Delta t}\mid \vx_t, t, \vc)
=
\gN\!\big(\vx_{t-\Delta t};\, \vmu_\theta(\vx_t,t,\vc),\, \sigma^2(t)\rmI\big)\textrm{,}
\end{equation}
where the specific forms of $\vmu_\theta$ and $\sigma(t)$ depend on the sampler.
The above generative process can be formulated as a finite-horizon Markov Decision Process (MDP) \citep{black2024training,fan2023reinforcement,liu2026flowgrpo}. To distinguish the discrete decision process from the underlying continuous-time flow, we use \(k\in\{1,\dots,K\}\) for the MDP state index and \(t\in[0,1]\) for the reverse-time variable of the flow. Let $0=t_K < t_{K-1} < \cdots < t_1 = 1$
be a discretization of reverse time, so that state \(k\) corresponds to flow time \(t_k\). The state at step \(k\) is
$\vs_k = (\vc, t_k, \vx_{t_k})$.
Note that \(t_k - t_{k+1}=\Delta t\).
For \(k=1,\dots,K-1\), the action is the next latent sample,
$\va_k = \vx_{t_{k+1}}$,
drawn from the sampler-induced policy
$\pi_\theta(\va_k \mid \vs_k)
=
\pi_\theta(\vx_{t_{k+1}} \mid \vx_{t_k}, t_k, \vc)$.
Given the sampled action, the transition is deterministic, with next state
$\vs_{k+1} = (\vc, t_{k+1}, \vx_{t_{k+1}})$.
The rollout starts from \(\vc\sim p(\vc)\) and \(\vx_{t_1}\sim\gN(\mathbf 0,\mathbf I)\), and terminates at \(k=K\) where \(t_K=0\).

After the full generative process, a scalar reward \(R(\vx_0,\vc)\) is provided. RL fine-tuning maximizes the expected terminal reward with a KL regularization term that penalizes deviation from the pretrained reference policy $\pi_{\mathrm{ref}}$: $\max_{\theta}\;
\E_{\vc \sim p(\vc),\, \tau \sim \pi_\theta}
\left[
R(\vx_0, \vc)
- \beta \sum_{k=1}^{K-1} \KL\!\big(\pi_\theta(\cdot \mid \vs_k) \,\|\, \pi_{\mathrm{ref}}(\cdot \mid \vs_k)\big)
\right]$,
where
$\tau = (\vs_1,\va_1,\vs_2,\va_2,\dots,\vs_{K-1},\va_{K-1},\vs_K)$
denotes a trajectory induced by \(\pi_\theta\) and $\beta \ge 0$ controls the regularization strength. This KL penalty discourages reward hacking and mitigates catastrophic forgetting of the pretrained model's capabilities.

Flow-GRPO \citep{liu2026flowgrpo} applies GRPO to the above MDP. Given a prompt $\vc$, the current policy generates a group of $G$ samples $\{\vx_0^i\}_{i=1}^G$. Their rewards are normalized within the group to obtain relative advantages:
$\hat{A}^i
=
\big(
R(\vx_0^i, \vc)
-
\operatorname{mean}(\{R(\vx_0^j, \vc)\}_{j=1}^G)
\big)/{
\operatorname{std}(\{R(\vx_0^j, \vc)\}_{j=1}^G)
}$.
In practice, each policy optimization iteration begins by rolling out a batch of data, which is then split into several minibatches for multiple gradient steps. This procedure introduces policy \emph{staleness}: after the first update, the optimizing policy has already diverged from the behavior policy that generated the data. To control this off-policy drift, a trust region mechanism is applied.
Following PPO~\citep{schulman2017proximal}, the policy is optimized using the clipped surrogate objective
\begin{equation}
\label{eq:flow_grpo_loss}
\gL^{\text{Flow-GRPO}}(\theta)
=
\E\left[
\frac{1}{G}\sum_{i=1}^{G}\frac{1}{K}\sum_{k=1}^K
\left(
\min\big(r_k^i(\theta)\hat{A}^i,\,
\mathrm{clip}(r_k^i(\theta),1-\epsilon,1+\epsilon)\hat{A}^i\big) 
\right)
\right]\textrm{,}
\end{equation}
where we omit the KL penalty term $\KL(\pi_\theta \,\|\, \pi_{\mathrm{ref}})$ for brevity, and the per-step importance ratio is defined as $r_k^i(\theta)
=
\frac{
p_\theta(\vx_{t_k-\Delta t}^i \mid \vx_{t_k}^i, \vc)
}{
p_{\theta_{\mathrm{old}}}(\vx_{t_k-\Delta t}^i \mid \vx_{t_k}^i, \vc)
}$.
Since both Flow-SDE and CPS define Gaussian per-step policies as in Eq.~(\ref{eq:abstract_policy}), the log-ratio admits the same closed-form expression:
\begin{equation}
\label{eq:log_ratio}
\log r_k^i(\theta)
=
\frac{
\|\vx_{t_k-\Delta t}^i - \vmu_{\theta_{\mathrm{old}}}\|^2
-
\|\vx_{t_k-\Delta t}^i - \vmu_\theta\|^2
}{
2\sigma^2(t_k)
}\textrm{.}
\end{equation}
Therefore, both samplers can be optimized within the same GRPO framework, differing only in the parameterization of the induced stochastic policy.

\section{Methodology}
\label{sec:method}

In this section, we first derive a policy improvement bound that justifies trust-region methods for flow models. Then, we show that ratio clipping is a noisy proxy for the true divergence constraint. Finally, we present Flow-DPPO, which leverages exact KL computation to enforce a deterministic divergence mask, yielding a tighter and variance-free trust-region constraint.

\subsection{Trust-Region Policy Optimization for Flow Matching Models}
\label{subsec:trust_region}

Inspired by \citet{schulman2017proximal,qi2026rethinking}, we adapt the trust region framework to the flow model fine-tuning setting defined in \Cref{subsec:RL_finetuning_for_flow_matching_models}. This setting differs from the classical discounted RL paradigm in two important ways. First, the problem is an undiscounted episodic task with a finite horizon of $K-1$ decision steps. Second, due to the terminal reward structure, advantages are estimated at the trajectory level rather than per step. These properties necessitate a tailored policy improvement guarantee. We follow the MDP defined in \Cref{subsec:RL_finetuning_for_flow_matching_models}.

\begin{theorem}[Performance Difference Identity for Flow Models]
\label{thm:flow_pdi}
In the finite-horizon flow model MDP with $K-1$ decision steps, let $J(\pi) = \E_{\vc \sim p(\vc),\, \tau \sim \pi}[R(\vx_0, \vc)]$ denote the expected reward. For any two policies $\pi_\theta$ and $\pi_{\theta_\mathrm{old}}$, the performance difference decomposes as: $J(\pi_\theta) - J(\pi_{\theta_\mathrm{old}}) = L'_{\theta_\mathrm{old}}(\pi_\theta) - \Delta(\pi_{\theta_\mathrm{old}}, \pi_\theta)$,
where the surrogate objective is
\begin{equation}
\label{eq:flow_surrogate}
L'_{\theta_\mathrm{old}}(\pi_\theta) = \E_{\tau \sim \pi_{\theta_\mathrm{old}}} \left[ R(\vx_0, \vc) \sum_{k=1}^{K-1} \left( \frac{\pi_\theta(\va_k \mid \vs_k)}{\pi_{\theta_\mathrm{old}}(\va_k \mid \vs_k)} - 1 \right) \right]\textrm{,}
\end{equation}
and the error term is
\begin{equation*}
\Delta(\pi_{\theta_\mathrm{old}}, \pi_\theta) = \E_{\tau \sim \pi_{\theta_\mathrm{old}}} \left[ R(\vx_0, \vc) \sum_{k=1}^{K-1} \left( \frac{\pi_\theta(\va_k \mid \vs_k)}{\pi_{\theta_\mathrm{old}}(\va_k \mid \vs_k)} - 1 \right) \left( 1 - \prod_{j=k+1}^{K-1} \frac{\pi_\theta(\va_j \mid \vs_j)}{\pi_{\theta_\mathrm{old}}(\va_j \mid \vs_j)} \right) \right]\textrm{.}
\end{equation*}
\end{theorem}
The surrogate $L'_{\theta_\mathrm{old}}(\pi_\theta)$ represents a first-order approximation to the true improvement, while the error term $\Delta$ captures higher-order interactions between per-step policy changes. To yield a practical optimization objective, we bound this error term.

\begin{theorem}[Policy Improvement Bound for Flow Models]
\label{thm:flow_improvement_bound}
In the finite-horizon flow model MDP with $K-1$ decision steps, the policy improvement is lower-bounded by:
\begin{equation}
\label{eq:flow_improvement_bound}
J(\pi_\theta) - J(\pi_{\theta_\mathrm{old}}) \ge L'_{\theta_\mathrm{old}}(\pi_\theta) - 2\xi (K\!-\!1)(K\!-\!2) \cdot {D_{\mathrm{TV}}^{\max}(\pi_{\theta_\mathrm{old}} \| \pi_\theta)}^2\textrm{,}
\end{equation}
where $D_{\mathrm{TV}}^{\max}(\pi_{\theta_\mathrm{old}} \| \pi_\theta) = \max_{\vs_k} D_{\mathrm{TV}}\big(\pi_{\theta_\mathrm{old}}(\cdot \mid \vs_k) \| \pi_\theta(\cdot \mid \vs_k)\big)$ is the maximum per-step Total Variation divergence, and $\xi = \max_{\vx_0, \vc} |R(\vx_0, \vc)|$ is the maximum absolute reward.
\end{theorem}

Please refer to Appendix~\ref{app:improvement_bound} for the detailed derivation; a tighter bound linear in $K$ is given in Appendix~\ref{app:tighter_linear_bound}.
This bound is structurally analogous to the policy improvement bound for LLMs derived in \citet{qi2026rethinking}. It provides a rigorous justification for trust-region methods in flow model fine-tuning: constraining the per-step divergence controls the penalty term and guarantees monotonic improvement.
Similar to TRPO~\citep{schulman2015trust}, we can solve the following constrained optimization problem to ensure stable learning:
\begin{equation}
\label{eq:flow_trpo_obj}
\max_{\pi_\theta} \quad L'_{\theta_\mathrm{old}}(\pi_\theta), \qquad
\text{s.t.} \quad D_{\mathrm{TV}}^{\max}(\pi_{\theta_\mathrm{old}} \| \pi_\theta) \le \delta\textrm{.}
\end{equation}

\begin{remark}[Exact Divergence in the Gaussian Setting]
\label{rmk:exact_divergence}
For the Gaussian per-step policies in Eq.~(\ref{eq:abstract_policy}), the TV divergence is a monotone function of the mean displacement:
\begin{equation}
    D_{\mathrm{TV}}\big(\pi_{\theta_\mathrm{old}}(\cdot \mid \vs_k) \| \pi_\theta(\cdot \mid \vs_k)\big) = 2\Phi\!\left(\frac{\|\vmu_{\theta_\mathrm{old}} - \vmu_\theta\|}{2\sigma(t_k)}\right) - 1\textrm{,}
\end{equation}
where $\Phi$ is the standard normal CDF. Constraining the TV divergence below a threshold is therefore equivalent to constraining $\|\vmu_{\theta_\mathrm{old}} - \vmu_\theta\|^2 \le \delta'$ for an appropriate $\delta'$, which is precisely the divergence measure that Flow-DPPO employs.
Moreover, the Pinsker inequality $D_{\mathrm{TV}}(p \| q)^2 \le \frac{1}{2}\KL(p \| q)$ ensures that our KL-based constraint also upper-bounds the TV divergence: when the per-step $\KL \le \delta$, we have $D_{\mathrm{TV}}^{\max} \le \sqrt{\delta/2}$. In the Gaussian equal-covariance case, the converse also holds since KL and TV are both monotone functions of $\|\vmu_{\theta_\mathrm{old}} - \vmu_\theta\|/\sigma$. Thus, our method is theoretically justified from both the KL and TV perspectives.
Unlike the LLM setting, where the discrete vocabulary requires approximate divergence computations~\citep{qi2026rethinking}, the Gaussian structure of flow models provides \emph{exact} per-step divergence at zero additional cost.
\end{remark}

\subsection{Pitfalls of Ratio Clipping in Flow-GRPO}
\label{subsec:ratio_limitations}

Flow-GRPO adopts PPO-style ratio clipping to enforce a trust region.
For consistency with the Flow-GRPO notation~\citep{liu2026flowgrpo}, in this and the following subsections we index denoising steps by the flow time $t$ (equivalently, $t = t_k$ in the MDP indexing of \Cref{subsec:RL_finetuning_for_flow_matching_models}).
The clipping condition $|r^i_t - 1| \le \epsilon$ is intended to prevent the new policy from deviating too far from the old one. However, the probability ratio is a fundamentally noisy proxy for the true policy divergence. By definition of the Total Variation divergence,
\begin{equation}
\label{eq:tv_ratio_relation}
D_{\mathrm{TV}}\big(\pi_{\theta_\text{old}}(\cdot \mid \vx_t) \,\|\, \pi_\theta(\cdot \mid \vx_t)\big)
= \frac{1}{2}\,\E_{\vx_{t-\Delta t} \sim \pi_{\theta_\text{old}}}\!\big[\,|r^i_t - 1|\,\big]\textrm{,}
\end{equation}
so each individual $|r^i_t - 1|$ is merely a \emph{single-sample Monte Carlo estimate} of $2\,D_{\mathrm{TV}}$. While the policy improvement bound (\Cref{thm:flow_improvement_bound}) calls for constraining $D_{\mathrm{TV}}^{\max}$, ratio clipping constrains this noisy per-sample surrogate instead. This issue was identified by \citet{qi2026rethinking} in the LLM setting; we now show that the resulting pathology is particularly severe in flow models due to the high-dimensional continuous action space.

Recall from Eq.~(\ref{eq:log_ratio}) that the log-ratio is:
\begin{equation}
\label{eq:log_ratio_expanded}
\log r^i_t(\theta) = \frac{\|\vx^i_{t-\Delta t} - \vmu_{\theta_\text{old}}\|^2 - \|\vx^i_{t-\Delta t} - \vmu_\theta\|^2}{2\sigma^2}\textrm{.}
\end{equation}
Since $\vx^i_{t-\Delta t}$ is sampled from $\gN(\vmu_{\theta_\text{old}}, \sigma^2\rmI)$, we can write $\vx^i_{t-\Delta t} = \vmu_{\theta_\text{old}} + \sigma \epsilonv$ where $\epsilonv \sim \gN(\bm{0}, \rmI)$. Substituting and letting $\vd = \vmu_\theta - \vmu_{\theta_\text{old}}$:
\begin{equation}
\label{eq:ratio_decomposition}
\log r^i_t(\theta) = \frac{\|\sigma \epsilonv\|^2 - \|\sigma\epsilonv - \vd\|^2}{2\sigma^2} = \frac{2\sigma\,\epsilonv^\top \vd - \|\vd\|^2}{2\sigma^2} = \frac{\epsilonv^\top \vd}{\sigma} - \frac{\|\vd\|^2}{2\sigma^2}\textrm{.}
\end{equation}
The first term, $\epsilonv^\top \vd / \sigma$, is a zero-mean random variable with variance $\|\vd\|^2 / \sigma^2$. This reveals that the log-ratio is dominated by \emph{noise}: the signal (the deterministic second term $-\|\vd\|^2 / (2\sigma^2)$) is exactly the negative of the KL divergence, but it is corrupted by a noise term whose standard deviation $\|\vd\| / \sigma$ is of the same order as the signal itself.
This analysis yields two key insights:
\begin{enumerate}[leftmargin=*]
    \item \textbf{High variance.} The ratio $r^i_t$ is inherently noisy due to the stochastic sample $\epsilonv$. Even when the true KL divergence $\|\vd\|^2 / (2\sigma^2)$ is moderate, individual ratio samples can be extreme (either very large or very small), triggering spurious clipping.
    \item \textbf{Noise-dependent clipping.} Whether an update is clipped depends heavily on the random noise $\epsilonv$ drawn during sampling, rather than the true policy divergence. Two trajectories with identical policy parameters but different noise realizations may receive entirely different clipping decisions.
\end{enumerate}

In contrast, the true KL divergence $\KL(\pi_{\theta_\text{old}} \| \pi_\theta) = \|\vd\|^2 / (2\sigma^2)$ is a \emph{deterministic} function of the policy parameters alone, unaffected by the sampling noise. This motivates our approach: replace the noisy ratio-based trust region with a direct divergence constraint. A detailed variance analysis is provided in Appendix~\ref{app:ratio_variance}.

\subsection{Divergence Proximal Policy Optimization for Flow Models}
\label{subsec:flow_dppo}
\vspace{-0.cm}

We now derive the divergence between old and new policies in the flow model setting and present our \textbf{Flow-DPPO} algorithm.

\textbf{Exact KL divergence.} Since both $\pi_{\theta_\text{old}}(\cdot \mid \vx_t)$ and $\pi_\theta(\cdot \mid \vx_t)$ are Gaussians with the same variance $\sigma^2 \rmI$ but different means, the KL divergence admits the closed form (see Appendix~\ref{app:kl_derivation} for derivation):
\begin{equation}
\label{eq:kl_flow}
\KL\big(\pi_{\theta_\text{old}}(\cdot \mid \vx_t) \,\|\, \pi_\theta(\cdot \mid \vx_t)\big) = \frac{\|\vmu_{\theta_\text{old}}(\vx_t, t) - \vmu_\theta(\vx_t, t)\|^2}{2\sigma^2}\textrm{.}
\end{equation}
For Flow-SDE (corresponding to Eq.~(\ref{eq:flow_sde})), $\sigma^2 = \sigma_t^2 \Delta t$, giving:
\begin{equation}
\label{eq:kl_sde}
\KL^{\text{SDE}}(\pi_{\theta_\text{old}} \| \pi_\theta) = \frac{\Delta t}{2}\left(\frac{\sigma_t(1-t)}{2t} + \frac{1}{\sigma_t}\right)^2 \|\vv_\theta(\vx_t, t) - \vv_{\theta_\text{old}}(\vx_t, t)\|^2\textrm{.}
\end{equation}
For CPS (corresponding to Eq.~(\ref{eq:cps})), with $\sigma_{\text{CPS}} = (t-\Delta t)\sin(\eta\pi/2)$:
\begin{equation}
\label{eq:kl_cps}
\KL^{\text{CPS}}(\pi_{\theta_\text{old}} \| \pi_\theta) = \frac{\|\vmu_\theta^{\text{CPS}}(\vx_t, t) - \vmu_{\theta_\text{old}}^{\text{CPS}}(\vx_t, t)\|^2}{2(t-\Delta t)^2 \sin^2(\eta\pi/2)}\textrm{.}
\end{equation}

\begin{remark}
In the LLM setting, DPPO \citep{qi2026rethinking} must approximate the true divergence via Binary or Top-K reductions of the vocabulary distribution, as computing exact TV or KL over $|\gV| > 100\text{K}$ tokens is memory-prohibitive. In flow models, the Gaussian policy structure yields \emph{exact} divergence at negligible cost, namely the squared difference between two forward passes of the velocity network. This makes divergence-based trust regions strictly more natural for flow models than for LLMs.
\end{remark}

\textbf{The Flow-DPPO mask.} We define the Flow-DPPO objective as:
\begin{equation}
\label{eq:flow_dppo_obj}
\gL^{\text{Flow-DPPO}}(\theta) = \E\left[\frac{1}{G}\sum_{i=1}^G \frac{1}{T}\sum_{t=0}^{T-1} \left( M^i_t \cdot r^i_t(\theta) \cdot \hat{A}^i - \beta \KL(\pi_\theta \| \pi_\text{ref}) \right)\right]\textrm{,}
\end{equation}
where the divergence-based mask is:
\begin{equation}
\label{eq:flow_dppo_mask}
M^i_t =
\begin{cases}
0, & \text{if } \big(\hat{A}^i > 0 \text{ and } r^i_t > 1 \text{ and } D_t > \delta\big) \\
   & \quad\; \text{or } \big(\hat{A}^i < 0 \text{ and } r^i_t < 1 \text{ and } D_t > \delta\big)\textrm{,} \\
1, & \text{otherwise}\textrm{,}
\end{cases}
\end{equation}
with $D_t \equiv \KL\big(\pi_{\theta_\text{old}}(\cdot \mid \vx^i_t) \,\|\, \pi_\theta(\cdot \mid \vx^i_t)\big)$ and $\delta$ a divergence threshold.

\textbf{Asymmetric design.} The mask in Eq.~(\ref{eq:flow_dppo_mask}) preserves the asymmetric structure that makes PPO effective. It only blocks updates that are \emph{already moving away} from the old policy:
\begin{itemize}[leftmargin=*]
    \item When $\hat{A}^i > 0$ and $r^i_t > 1$: the gradient is pushing the policy \emph{further} from $\theta_\text{old}$ (increasing an already-increased action probability). The mask blocks this if the divergence exceeds $\delta$.
    \item When $\hat{A}^i < 0$ and $r^i_t < 1$: the gradient is decreasing an already-decreased action probability, again moving \emph{away} from the old policy. The mask blocks this if divergence exceeds $\delta$.
    \item In all other cases ($\hat{A}^i > 0, r^i_t < 1$ or $\hat{A}^i < 0, r^i_t > 1$): the gradient is moving the policy \emph{towards} the old policy. These beneficial updates are \emph{never} blocked, regardless of the divergence level.
\end{itemize}
This asymmetry ensures that the trust region constraint does not impede recovery: when the policy has drifted too far, corrective updates remain uninhibited. We provide a  justification of this directional condition and discuss refined mask variants in Appendix~\ref{app:future_work}.

\section{Experiments}
\label{sec:experiments}

\textbf{Models and Baselines.}
We employ Stable Diffusion 3.5 Medium~\citep{sd3,sd35} (SD3.5), FLUX2-klein-base-9B~\citep{flux2klein} (FLUX2-9B) and FLUX.1-dev~\citep{flux1dev} as base models to cover diverse architectures and scales.
We compare our method against four competitive baselines:
Flow-GRPO~\citep{liu2026flowgrpo}, Flow-CPS~\citep{wang2025cps}, GRPO-Guard~\citep{grpoguard} and Diffusion-NFT~\citep{diffusionnft}.
Specifically, we evaluate two variants of our approach: Flow-DPPO (using SDE sampling from Flow-GRPO) and Flow-DPPO+CPS (using CPS-scheduled SDE sampling).
Detailed configurations are deferred to Appendix~\ref{sec:appendix_experimental_details}.

\textbf{Metrics and Datasets.}
GenEval2~\citep{geneval2} and PickScore~\citep{kirstain2023pick} are selected as in-domain and out-of-domain (OOD) datasets, respectively.
For GenEval2, we follow the official template to generate 20k synthetic training prompts and evaluate on the 800 officially released prompts.
To monitor catastrophic forgetting under distribution shifts, we track PickScore~\citep{kirstain2023pick}, CLIP~\citep{clip} score, and HPSv2~\citep{wu2023human} during training.
We report results for both single-reward optimization (GenEval2 only) and multi-reward training,
where GDPO~\citep{gdpo} aggregates advantages with equal reward weights.

\definecolor{ourrow}{HTML}{E8F1FF}
\begin{table*}[t]
\centering
\setlength{\tabcolsep}{4pt}
\caption{Performance comparison on SD3.5 and FLUX2-9B. The training is applied on the In-Domain (GenEval2). The Out-of-Domain (PickScore) prompts are only used for evaluation. The corresponding training curves are in \Cref{fig:train_row_sd35_multi,fig:train_row_flux_multi}. The full version including single-reward training is in \Cref{tab:appendix_geneval2_baselines}.}
\label{tab:main_results}
\scriptsize
\begin{tabular}{lccccccc}
\toprule
& \multicolumn{4}{c}{\textbf{In-Domain (GenEval2)}} & \multicolumn{3}{c}{\textbf{Out-of-Domain (PickScore)}} \\
\cmidrule(lr){2-5} \cmidrule(lr){6-8}
\textbf{Model} & GenEval2 & CLIP & PickScore & HPSv2 & CLIP & PickScore & HPSv2 \\
\midrule
\multicolumn{8}{l}{\textit{Pretrained baselines (before RL)}} \\
\quad SD3.5-medium & 12.4 & 0.250 & 21.00 & 0.213 & 0.244 & 19.99 & 0.210 \\
\quad FLUX2-klein-base-9B & 25.4 & 0.281 & 20.92 & 0.228 & 0.254 & 20.05 & 0.230 \\
\quad FLUX.1-dev & 23.3 & 0.297 & 23.26 & 0.315 & 0.276 & 21.91 & 0.304 \\
\midrule
\multicolumn{8}{l}{\textit{SD3.5-medium, multi-reward RL fine-tuning}} \\
\quad Flow-GRPO & 39.9 & 0.358 & 25.09 & 0.399 & \underline{0.273} & 22.07 & 0.349 \\
\quad Flow-CPS & 44.6 & \underline{0.359} & 25.51 & 0.407 & 0.265 & 22.08 & 0.343 \\
\quad GRPO-Guard & 47.8 & 0.353 & \underline{25.64} & \underline{0.409} & 0.272 & 22.32 & 0.354 \\
\quad Diffusion-NFT & 42.5 & 0.334 & 25.30 & 0.394 & 0.269 & \underline{22.52} & 0.355 \\
\rowcolor{ourrow} \quad Flow-DPPO & \underline{48.1} & 0.345 & 25.63 & 0.409 & 0.273 & \textbf{22.58} & \underline{0.360} \\
\rowcolor{ourrow} \quad Flow-DPPO + CPS & \textbf{51.6} & \textbf{0.369} & \textbf{25.72} & \textbf{0.415} & \textbf{0.279} & 22.51 & \textbf{0.361} \\
\midrule
\multicolumn{8}{l}{\textit{FLUX2-klein-base-9B, multi-reward RL fine-tuning}} \\
\quad Flow-GRPO & 46.8 & 0.371 & 25.61 & 0.412 & 0.277 & 22.62 & 0.357 \\
\quad Flow-CPS & 47.1 & 0.361 & 25.70 & 0.416 & 0.276 & 22.85 & 0.364 \\
\quad GRPO-Guard & 49.0 & \underline{0.375} & 25.27 & 0.411 & 0.269 & 21.99 & 0.349 \\
\quad Diffusion-NFT & 47.3 & 0.336 & 24.87 & 0.389 & 0.274 & 22.47 & 0.351 \\
\rowcolor{ourrow} \quad Flow-DPPO & \textbf{57.7} & 0.364 & \underline{25.76} & \underline{0.418} & \underline{0.282} & \underline{22.90} & \underline{0.368} \\
\rowcolor{ourrow} \quad Flow-DPPO + CPS & \underline{55.2} & \textbf{0.386} & \textbf{26.15} & \textbf{0.427} & \textbf{0.287} & \textbf{22.97} & \textbf{0.370} \\
\bottomrule
\end{tabular}
\end{table*}

\begin{figure*}[htbp]
  \centering
  \includegraphics[width=\linewidth]{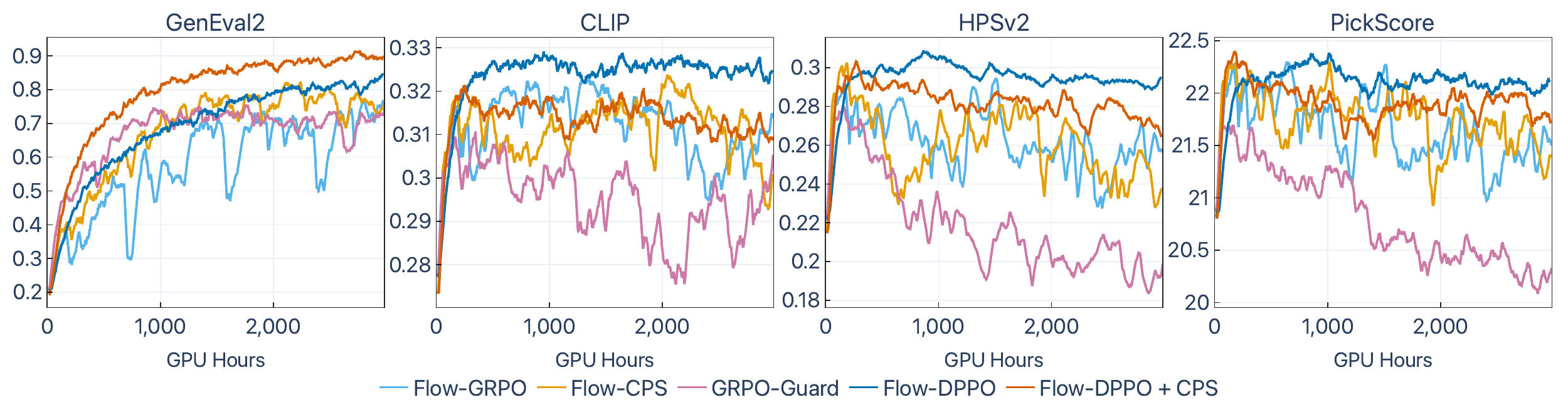}
  \caption{Training curves on FLUX2-9B for single-reward setting. Flow-DPPO variants achieve state-of-the-art performance and less catastrophic forgetting on out-of-domain rewards.}
  \label{fig:train_row_flux_single_nocfg}
\end{figure*}

\subsection{Main results}

\textbf{Performance and Generalization.}
As summarized in Table~\ref{tab:main_results},
Flow-DPPO variants consistently outperform all baselines across both base models and all evaluation metrics,
with particularly substantial gains in the GenEval2 reward.
In the single-reward setting (optimizing GenEval2 only),
Figure~\ref{fig:train_row_flux_single_nocfg} demonstrates that our proposed variants not only achieve superior performance on FLUX2-9B compared to baselines but also exhibit a more stable training trajectory.
These empirical advantages persist across SD3.5 (Figure~\ref{fig:train_row_sd35_single}) and FLUX.1-dev (Figure~\ref{fig:train_row_flux1dev}).

We attribute this superiority to the precise divergence-based mask in Flow-DPPO.
By mitigating the influence of samples falling outside the trust region, which are susceptible to reward hacking,
Flow-DPPO maintains a more robust optimization gradient.
This constraint prevents the model from excessively exploiting individual rewards at the expense of others,
thereby achieving a superior balance across multiple optimization objectives and fostering stable convergence.
This is further corroborated by the multi-reward training curves in Figure~\ref{fig:train_row_sd35_multi},
where Flow-DPPO variants consistently outperform all baselines across most metrics on SD3.5,
without sacrificing any individual objective.

\textbf{Out-of-domain Behavior and Catastrophic Forgetting.}
To investigate catastrophic forgetting,
we analyze OOD metrics (PickScore, CLIP, and HPSv2) and the KL divergence from the pre-trained model.
As illustrated in Figure~\ref{fig:train_row_flux_single_nocfg},
OOD metrics initially increase across all methods as RL optimization drives the model toward higher visual quality.
However, as training progresses, these metrics decline,
indicating that the model overfits the in-domain reward (GenEval2) at the expense of OOD knowledge.
Notably, Flow-DPPO variants exhibit significantly less OOD degradation,
suggesting that catastrophic forgetting is effectively mitigated.
Qualitative results in Figure~\ref{fig:case_visualization} further support this,
demonstrating that our methods better preserve visual fidelity on OOD prompts.
Consistently, Table~\ref{tab:kl_final_step} shows that Flow-DPPO variants maintain a lower KL divergence in most settings.
This reduced distribution drift aligns with OOD metric trends, collectively indicating stronger resistance to reward hacking and forgetting.
Ultimately, these results highlight that the divergence-based mask acts as a safety boundary,
allowing the model to learn from rewards without losing its original generative quality or falling into distribution collapse.

\vspace{-0.175cm}
\subsection{Analysis}
\vspace{-0.1cm}

\textbf{Asymmetric Masking and Divergence Threshold.}
We investigate the impact of the divergence threshold and asymmetric masking in Flow-DPPO using SD3.5 with CPS sampling (Figure~\ref{fig:sd35_kl_adv_ablation}).
Without asymmetric masking,
the training process collapses as the trust-region regularization becomes ineffective;
specifically, samples falling outside the trust region are largely ignored, preventing optimization progress.
Conversely, asymmetric masking constrains these samples back within the trust region,
thereby stabilizing the trajectory.
Regarding the divergence threshold,
a looser threshold ($10^{-5}$) results in diminished stability and suboptimal convergence.
A tighter threshold ($10^{-7}$) initially slows down learning but fosters superior stability and slightly better final performance due to more rigorous trust-region enforcement.

\textbf{Multi-epoch Training and Sample Efficiency.}
Given the high computational cost of rollouts, we investigate how sample reuse frequency affects optimization efficiency on SD3.5.
Specifically, we vary two factors:
(i) the number of groups per rollout,
and (ii) the number of training epochs per rollout (inner loops).
The latter determines the reuse frequency of each sample.
For instance, two inner loops imply that each rollout batch is utilized for two consecutive gradient steps.

\begin{figure*}[t]
  \begin{minipage}[t]{0.49\textwidth}
    \centering
    \captionof{figure}{Asymmetric masking ablation on SD3.5 with single-reward on GenEval2.}
    \label{fig:sd35_kl_adv_ablation}
    \includegraphics[width=\linewidth]{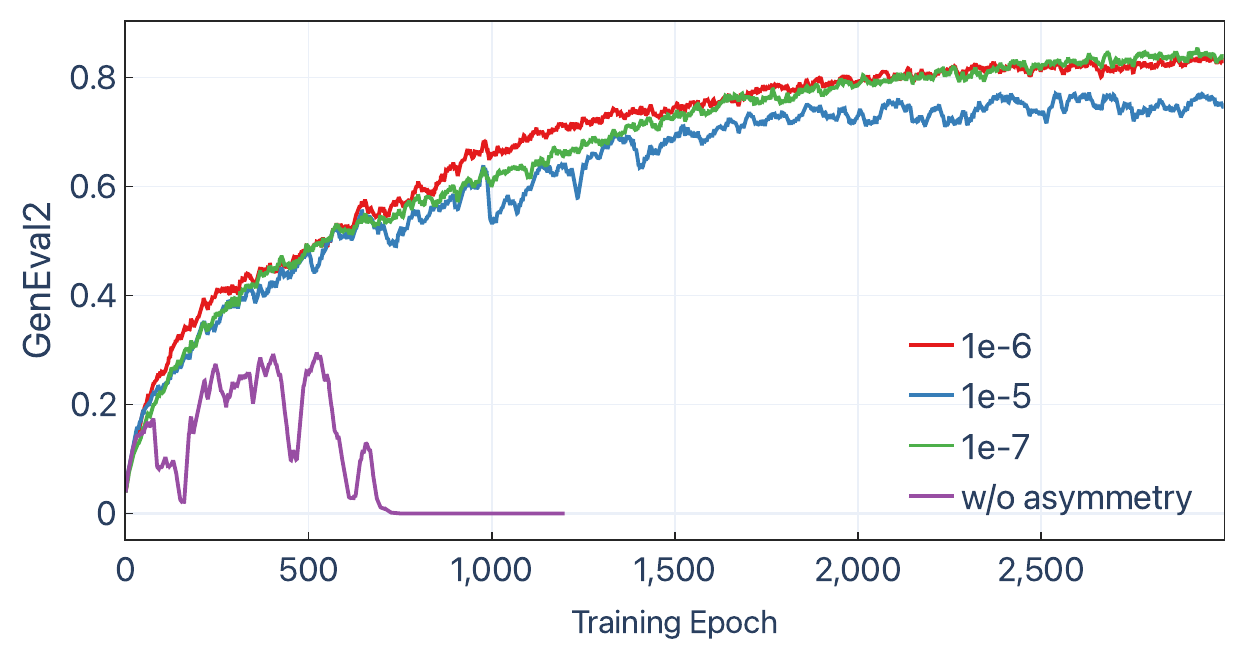}
  \end{minipage}\hspace{0.02\textwidth}%
  \begin{minipage}[t]{0.49\textwidth}
    \centering
    \captionof{table}{KL divergence (\(\times 10^{-3}\)) between the RL fine-tuned model and the pre-trained reference.
    Lower is better.
    Full curves in Figure~\ref{fig:kl_div}.}
    \label{tab:kl_final_step}
    \scriptsize
    \setlength{\tabcolsep}{3pt}
    \begin{tabular}{lccccc}
    \toprule
    & \multicolumn{3}{c}{\textbf{FLUX2-9B}} & \multicolumn{2}{c}{\textbf{SD3.5}} \\
    \cmidrule(lr){2-4} \cmidrule(lr){5-6}
    \textbf{Method} & Single & Multi & +CFG & Single & Multi \\
    \midrule
    \multicolumn{6}{l}{\textit{Flow-SDE schedule}} \\
    \quad Flow-GRPO & \underline{0.77} & \underline{0.79} & \underline{1.36} & 2.34 & 3.81 \\
    \quad GRPO-Guard & 1.07 & 1.01 & 1.63 & \underline{2.05} & \underline{3.33} \\
    \rowcolor{ourrow} \quad Flow-DPPO & \textbf{0.17} & \textbf{0.49} & \textbf{0.51} & \textbf{1.16} & \textbf{2.49} \\
    \midrule
    \multicolumn{6}{l}{\textit{CPS schedule}} \\
    \quad Flow-CPS & \textbf{0.24} & \underline{1.66} & \underline{1.51} & \underline{2.41} & \underline{3.18} \\
    \rowcolor{ourrow} \quad Flow-DPPO + CPS & \underline{0.68} & \textbf{0.70} & \textbf{0.83} & \textbf{1.60} & \textbf{2.52} \\
    \bottomrule
    \end{tabular}
  \end{minipage}
\end{figure*}

\begin{figure*}[t]
  \centering
  \includegraphics[width=\linewidth]{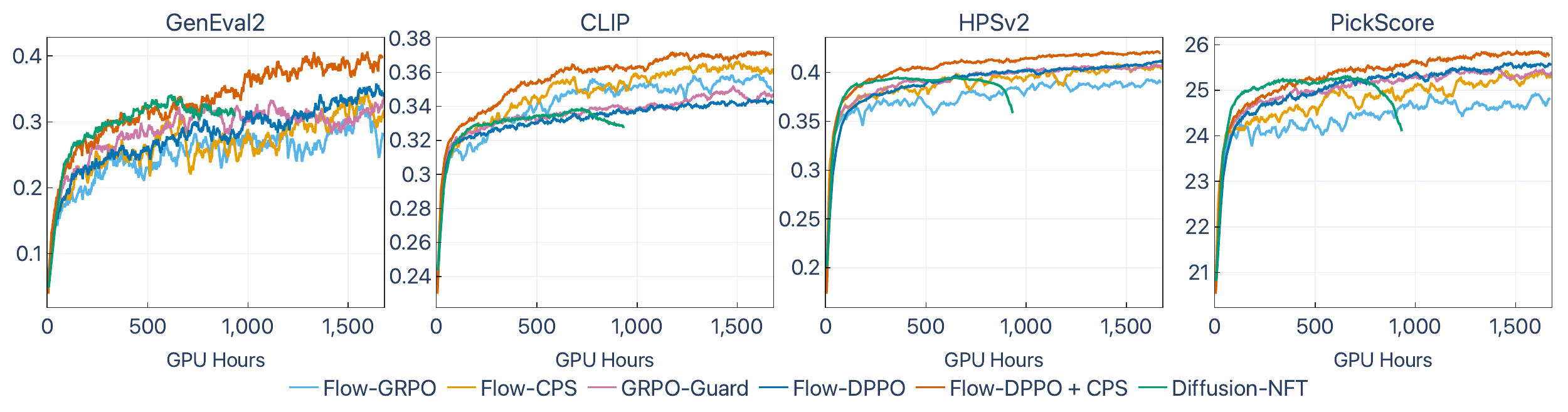}
  \caption{Training curves on SD3.5 for multi-reward setting. Flow-DPPO variants consistently outperform all baselines across all metrics.}
  \label{fig:train_row_sd35_multi}
\end{figure*}


\newcommand{\CaseMethodOne}{FLUX2-9B}
\newcommand{\CaseMethodTwo}{Flow-GRPO}
\newcommand{\CaseMethodThree}{Flow-CPS}
\newcommand{\CaseMethodFour}{GRPO-Guard}
\newcommand{\CaseMethodFive}{\textbf{Flow-DPPO}}
\newcommand{\CaseMethodSix}{\textbf{Flow-DPPO+CPS}}

\newcommand{\CaseCellOneTwo}{Single-reward/geneval2/idx026_detailed-prompt-a-small-stone-pig-statue-with-a-rough-textured-surface-sits-quie/01_Flux2-klein-base-9B.jpg}
\newcommand{\CaseCellOneThree}{Single-reward/geneval2/idx001_five-colorful-plastic-bicycles-are-arranged-in-a-neat-row-on-a-flat-surface-the-/01_Flux2-klein-base-9B.jpg}
\newcommand{\CaseCellOneFive}{Single-reward/pickscore/idx008_magnificent-body-shot-of-a-sun-elf-girl-20-years-old-holding-bow-and-arrow-towar/01_Flux2-klein-base-9B.jpg}
\newcommand{\CaseCellOneSix}{Single-reward/pickscore/idx007_people-at-the-barbecue-in-brazil-hd-canon-eos-5d-mark-iv-dslr/01_Flux2-klein-base-9B.jpg}

\newcommand{\CaseCellTwoTwo}{Single-reward/geneval2/idx026_detailed-prompt-a-small-stone-pig-statue-with-a-rough-textured-surface-sits-quie/02_Flow-GRPO.jpg}
\newcommand{\CaseCellTwoThree}{Single-reward/geneval2/idx001_five-colorful-plastic-bicycles-are-arranged-in-a-neat-row-on-a-flat-surface-the-/02_Flow-GRPO.jpg}
\newcommand{\CaseCellTwoFive}{Single-reward/pickscore/idx008_magnificent-body-shot-of-a-sun-elf-girl-20-years-old-holding-bow-and-arrow-towar/02_Flow-GRPO.jpg}
\newcommand{\CaseCellTwoSix}{Single-reward/pickscore/idx007_people-at-the-barbecue-in-brazil-hd-canon-eos-5d-mark-iv-dslr/02_Flow-GRPO.jpg}

\newcommand{\CaseCellThreeTwo}{Single-reward/geneval2/idx026_detailed-prompt-a-small-stone-pig-statue-with-a-rough-textured-surface-sits-quie/03_Flow-CPS.jpg}
\newcommand{\CaseCellThreeThree}{Single-reward/geneval2/idx001_five-colorful-plastic-bicycles-are-arranged-in-a-neat-row-on-a-flat-surface-the-/03_Flow-CPS.jpg}
\newcommand{\CaseCellThreeFive}{Single-reward/pickscore/idx008_magnificent-body-shot-of-a-sun-elf-girl-20-years-old-holding-bow-and-arrow-towar/03_Flow-CPS.jpg}
\newcommand{\CaseCellThreeSix}{Single-reward/pickscore/idx007_people-at-the-barbecue-in-brazil-hd-canon-eos-5d-mark-iv-dslr/03_Flow-CPS.jpg}

\newcommand{\CaseCellFourTwo}{Single-reward/geneval2/idx026_detailed-prompt-a-small-stone-pig-statue-with-a-rough-textured-surface-sits-quie/06_GRPO-Guard.jpg}
\newcommand{\CaseCellFourThree}{Single-reward/geneval2/idx001_five-colorful-plastic-bicycles-are-arranged-in-a-neat-row-on-a-flat-surface-the-/06_GRPO-Guard.jpg}
\newcommand{\CaseCellFourFive}{Single-reward/pickscore/idx008_magnificent-body-shot-of-a-sun-elf-girl-20-years-old-holding-bow-and-arrow-towar/06_GRPO-Guard.jpg}
\newcommand{\CaseCellFourSix}{Single-reward/pickscore/idx007_people-at-the-barbecue-in-brazil-hd-canon-eos-5d-mark-iv-dslr/06_GRPO-Guard.jpg}

\newcommand{\CaseCellFiveTwo}{Single-reward/geneval2/idx026_detailed-prompt-a-small-stone-pig-statue-with-a-rough-textured-surface-sits-quie/04_Flow-DPPO.jpg}
\newcommand{\CaseCellFiveThree}{Single-reward/geneval2/idx001_five-colorful-plastic-bicycles-are-arranged-in-a-neat-row-on-a-flat-surface-the-/04_Flow-DPPO.jpg}
\newcommand{\CaseCellFiveFive}{Single-reward/pickscore/idx008_magnificent-body-shot-of-a-sun-elf-girl-20-years-old-holding-bow-and-arrow-towar/04_Flow-DPPO.jpg}
\newcommand{\CaseCellFiveSix}{Single-reward/pickscore/idx007_people-at-the-barbecue-in-brazil-hd-canon-eos-5d-mark-iv-dslr/04_Flow-DPPO.jpg}

\newcommand{\CaseCellSixTwo}{Single-reward/geneval2/idx026_detailed-prompt-a-small-stone-pig-statue-with-a-rough-textured-surface-sits-quie/05_Flow-DPPO_+_CPS.jpg}
\newcommand{\CaseCellSixThree}{Single-reward/geneval2/idx001_five-colorful-plastic-bicycles-are-arranged-in-a-neat-row-on-a-flat-surface-the-/05_Flow-DPPO_+_CPS.jpg}
\newcommand{\CaseCellSixFive}{Single-reward/pickscore/idx008_magnificent-body-shot-of-a-sun-elf-girl-20-years-old-holding-bow-and-arrow-towar/05_Flow-DPPO_+_CPS.jpg}
\newcommand{\CaseCellSixSix}{Single-reward/pickscore/idx007_people-at-the-barbecue-in-brazil-hd-canon-eos-5d-mark-iv-dslr/05_Flow-DPPO_+_CPS.jpg}

\par\vspace{0.6em}
\noindent\begin{minipage}{\textwidth}
  {\centering
  \begin{minipage}[t]{0.495\textwidth}
    \centering
    \includegraphics[width=\linewidth]{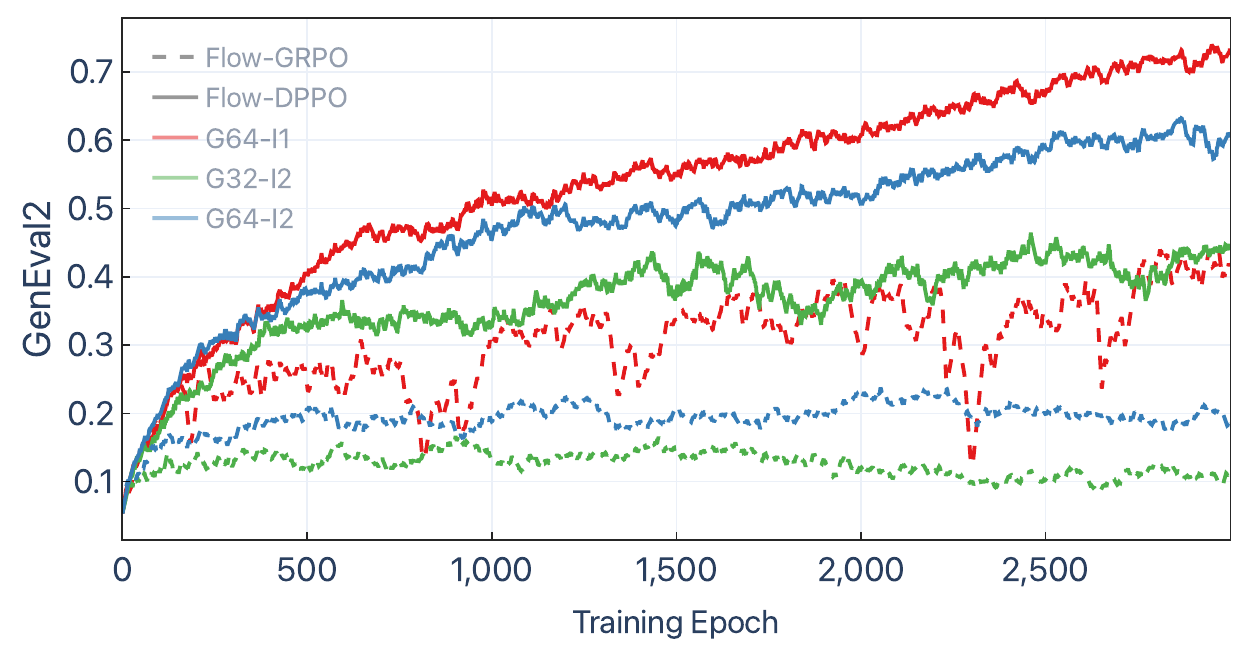}
  \end{minipage}\hspace{0.01\textwidth}%
  \begin{minipage}[t]{0.495\textwidth}
    \centering
    \includegraphics[width=\linewidth]{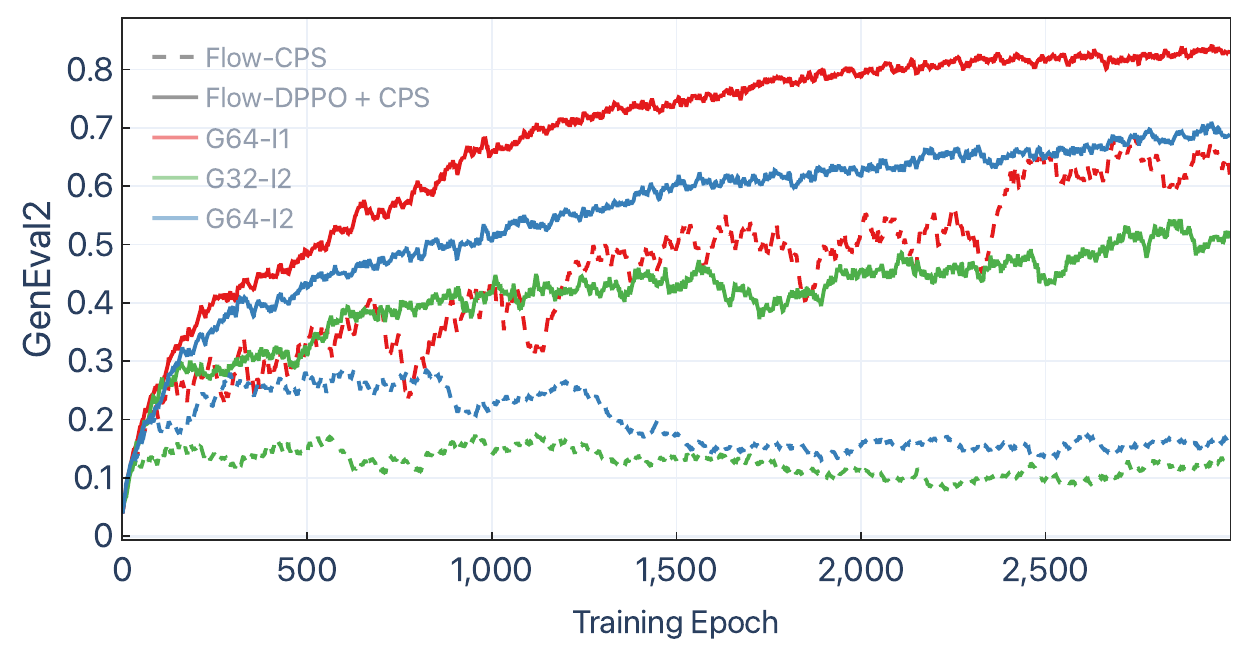}
  \end{minipage}\par}
  \refstepcounter{figure}\label{fig:sd35_multiepoch}%
  \makeatletter
  \@makecaption{Figure~\thefigure}{Multi-epoch training on SD3.5 (Left: Flow-SDE, Right: CPS). Flow-DPPO variants show consistent long-term gains under multi-epoch training (G64-I2 and G32-I2), while baselines plateau or even degrade.}%
  \makeatother

  \vspace{1.0em}
  \nopagebreak

  {\centering
  \setlength{\tabcolsep}{0pt}
  \renewcommand{\arraystretch}{0}
  \newcommand{\CaseImgT}[1]{\includegraphics[width=0.1445\textwidth]{figs/flux2_geneval2_pickscore_by_category_bundle/#1}}
  \newcommand{\CasePromptLabel}[1]{\parbox[c]{0.11\textwidth}{\centering\fontsize{7}{8}\selectfont\textit{#1}}}
  \newcommand{\CaseHead}[1]{{\fontsize{6.3}{7.5}\selectfont #1}}
  \begin{tabular}{>{\centering\arraybackslash}m{0.115\textwidth}*{6}{>{\centering\arraybackslash}m{0.1445\textwidth}}}
    & \CaseHead{\CaseMethodOne} & \CaseHead{\CaseMethodTwo} & \CaseHead{\CaseMethodThree}
    & \CaseHead{\CaseMethodFour} & \CaseHead{\CaseMethodFive} & \CaseHead{\CaseMethodSix} \\[2pt]
    \multicolumn{7}{l}{\scriptsize\textbf{In-Domain}} \\[1pt]
    \CasePromptLabel{a stone pig in background, two black cats in front of the pig, and six yellow horses in front}
    & \CaseImgT{\CaseCellOneTwo} & \CaseImgT{\CaseCellTwoTwo} & \CaseImgT{\CaseCellThreeTwo}
    & \CaseImgT{\CaseCellFourTwo} & \CaseImgT{\CaseCellFiveTwo} & \CaseImgT{\CaseCellSixTwo} \\
    \CasePromptLabel{five colorful bicycles below, two stone guitars above them, and a penguin at the highest point}
    & \CaseImgT{\CaseCellOneThree} & \CaseImgT{\CaseCellTwoThree} & \CaseImgT{\CaseCellThreeThree}
    & \CaseImgT{\CaseCellFourThree} & \CaseImgT{\CaseCellFiveThree} & \CaseImgT{\CaseCellSixThree} \\
    \multicolumn{7}{l}{\scriptsize\textbf{Out-of-Domain}} \\[1pt]
    \CasePromptLabel{a sun elf with a bow, facing the camera, in a jungle waterfall scene}
    & \CaseImgT{\CaseCellOneFive} & \CaseImgT{\CaseCellTwoFive} & \CaseImgT{\CaseCellThreeFive}
    & \CaseImgT{\CaseCellFourFive} & \CaseImgT{\CaseCellFiveFive} & \CaseImgT{\CaseCellSixFive} \\
    \CasePromptLabel{people at a barbecue in Brazil, captured in HD, Canon EOS 5D Mark IV DSLR}
    & \CaseImgT{\CaseCellOneSix} & \CaseImgT{\CaseCellTwoSix} & \CaseImgT{\CaseCellThreeSix}
    & \CaseImgT{\CaseCellFourSix} & \CaseImgT{\CaseCellFiveSix} & \CaseImgT{\CaseCellSixSix} \\
  \end{tabular}\par}
  \refstepcounter{figure}\label{fig:case_visualization}%
  \makeatletter
  \@makecaption{Figure~\thefigure}{Qualitative comparison on FLUX2-9B with single-reward setting and controlled seeds for each prompt at the same training iteration.
Flow-DPPO and Flow-DPPO + CPS retain competitive in-domain performance with less reward hacking while exhibiting notably less catastrophic forgetting on out-of-domain prompts.}%
  \makeatother
\end{minipage}\par
\vspace{0.6em}

While our main experiments use 64 groups with 1 inner loop (G64-I1),
we further explore two efficiency-oriented settings:
G32-I2 (half the rollout computation with samples reused twice) and G64-I2 (standard rollout computation with doubled training intensity).
As shown in Figure~\ref{fig:sd35_multiepoch},
baseline methods (Flow-GRPO, Flow-CPS) struggle to achieve sustained gains under multi-epoch training,
often leading to performance plateaus or degradation.
In contrast, Flow-DPPO variants successfully reuse rollout samples across multiple updates,
yielding consistent long-term performance improvements.
This advantage stems from the divergence-based mask, which constrains updates within the trust region, ensuring efficient sample utilization.
This offers a promising direction for scenarios where rollouts are computationally expensive, such as long-video generation.
\vspace{-0.175cm}
\section{Conclusion}
\label{sec:conclusion}
\vspace{-0.175cm}

We show ratio clipping in flow models is a noisy, biased proxy for divergence.
To address this, we propose a divergence-based mask using the \emph{exact} KL at \emph{zero extra cost}.
Across multiple base models, sampling schedules, and reward objectives,
Flow-DPPO consistently achieves superior performance than baselines in terms of reward optimization and catastrophic forgetting.
Furthermore, Flow-DPPO enables stable multi-epoch training where ratio clipping degrades,
offering a promising direction for scenarios with expensive rollouts, such as long-video generation.

\bibliography{citations}

@inproceedings{
karras2022elucidating,
title={Elucidating the Design Space of Diffusion-Based Generative Models},
author={Tero Karras and Miika Aittala and Timo Aila and Samuli Laine},
booktitle={Advances in Neural Information Processing Systems},
year={2022}
}

@inproceedings{
lu2022dpm,
title={{DPM}-Solver: A Fast {ODE} Solver for Diffusion Probabilistic Model Sampling in Around 10 Steps},
author={Cheng Lu and Yuhao Zhou and Fan Bao and Jianfei Chen and Chongxuan Li and Jun Zhu},
booktitle={Advances in Neural Information Processing Systems},
year={2022},
}

@inproceedings{
song2021scorebased,
title={Score-Based Generative Modeling through Stochastic Differential Equations},
author={Yang Song and Jascha Sohl-Dickstein and Diederik P Kingma and Abhishek Kumar and Stefano Ermon and Ben Poole},
booktitle={International Conference on Learning Representations},
year={2021}
}

@inproceedings{
lipman2023flow,
title={Flow Matching for Generative Modeling},
author={Yaron Lipman and Ricky T. Q. Chen and Heli Ben-Hamu and Maximilian Nickel and Matthew Le},
booktitle={The Eleventh International Conference on Learning Representations },
year={2023}
}

@inproceedings{
song2021denoising,
title={Denoising Diffusion Implicit Models},
author={Jiaming Song and Chenlin Meng and Stefano Ermon},
booktitle={International Conference on Learning Representations},
year={2021}
}

@inproceedings{
liu2023flow,
title={Flow Straight and Fast: Learning to Generate and Transfer Data with Rectified Flow},
author={Xingchao Liu and Chengyue Gong and qiang liu},
booktitle={The Eleventh International Conference on Learning Representations },
year={2023}
}

@inproceedings{
albergo2023building,
title={Building Normalizing Flows with Stochastic Interpolants},
author={Michael Samuel Albergo and Eric Vanden-Eijnden},
booktitle={The Eleventh International Conference on Learning Representations },
year={2023}
}

@article{liu2026flowgrpo,
  title={Flow-grpo: Training flow matching models via online rl},
  author={Liu, Jie and Liu, Gongye and Liang, Jiajun and Li, Yangguang and Liu, Jiaheng and Wang, Xintao and Wan, Pengfei and Zhang, Di and Ouyang, Wanli},
  journal={arXiv preprint arXiv:2505.05470},
  year={2025}
}

@article{wang2025cps,
  title={Coefficients-Preserving Sampling for Reinforcement Learning with Flow Matching},
  author={Wang, Feng and Yu, Zihao},
  journal={arXiv preprint arXiv:2509.05952},
  year={2025}
}

@inproceedings{wallace2024diffusion,
  title={Diffusion model alignment using direct preference optimization},
  author={Wallace, Bram and Dang, Meihua and Rafailov, Rafael and Zhou, Linqi and Lou, Aaron and Purushwalkam, Senthil and Ermon, Stefano and Xiong, Caiming and Joty, Shafiq and Naik, Nikhil},
  booktitle={Proceedings of the IEEE/CVF Conference on Computer Vision and Pattern Recognition},
  pages={8228--8238},
  year={2024}
}

@inproceedings{
black2024training,
title={Training Diffusion Models with Reinforcement Learning},
author={Kevin Black and Michael Janner and Yilun Du and Ilya Kostrikov and Sergey Levine},
booktitle={The Twelfth International Conference on Learning Representations},
year={2024}
}

@article{xue2025dancegrpo,
  title={Dancegrpo: Unleashing grpo on visual generation},
  author={Xue, Zeyue and Wu, Jie and Gao, Yu and Kong, Fangyuan and Zhu, Lingting and Chen, Mengzhao and Liu, Zhiheng and Liu, Wei and Guo, Qiushan and Huang, Weilin and others},
  journal={arXiv preprint arXiv:2505.07818},
  year={2025}
}

@inproceedings{
fan2023reinforcement,
title={Reinforcement Learning for Fine-tuning Text-to-Image Diffusion Models},
author={Ying Fan and Olivia Watkins and Yuqing Du and Hao Liu and Moonkyung Ryu and Craig Boutilier and Pieter Abbeel and Mohammad Ghavamzadeh and Kangwook Lee and Kimin Lee},
booktitle={Thirty-seventh Conference on Neural Information Processing Systems},
year={2023}
}

@article{schulman2017proximal,
  title={Proximal policy optimization algorithms},
  author={Schulman, John and Wolski, Filip and Dhariwal, Prafulla and Radford, Alec and Klimov, Oleg},
  journal={arXiv preprint arXiv:1707.06347},
  year={2017}
}

@inproceedings{schulman2015trust,
  title={Trust region policy optimization},
  author={Schulman, John and Levine, Sergey and Abbeel, Pieter and Jordan, Michael and Moritz, Philipp},
  booktitle={International conference on machine learning},
  pages={1889--1897},
  year={2015},
  organization={PMLR}
}

@article{grpo,
  title={Deepseekmath: Pushing the limits of mathematical reasoning in open language models},
  author={Shao, Zhihong and Wang, Peiyi and Zhu, Qihao and Xu, Runxin and Song, Junxiao and Bi, Xiao and Zhang, Haowei and Zhang, Mingchuan and Li, YK and Wu, Y and others},
  journal={arXiv preprint arXiv:2402.03300},
  year={2024}
}

@article{qi2026rethinking,
  title={Rethinking the Trust Region in LLM Reinforcement Learning},
  author={Qi, Penghui and Zhou, Xiangxin and Liu, Zichen and Pang, Tianyu and Du, Chao and Lin, Min and Lee, Wee Sun},
  journal={arXiv preprint arXiv:2602.04879},
  year={2026}
}

@inproceedings{kakade2002approximately,
  title={Approximately optimal approximate reinforcement learning},
  author={Kakade, Sham and Langford, John},
  booktitle={Proceedings of the nineteenth international conference on machine learning},
  pages={267--274},
  year={2002}
}

@article{ouyang2022training,
  title={Training language models to follow instructions with human feedback},
  author={Ouyang, Long and Wu, Jeffrey and Jiang, Xu and Almeida, Diogo and Wainwright, Carroll and Mishkin, Pamela and Zhang, Chong and Agarwal, Sandhini and Slama, Katarina and Ray, Alex and others},
  journal={Advances in neural information processing systems},
  volume={35},
  pages={27730--27744},
  year={2022}
}

@article{rafailov2023direct,
  title={Direct preference optimization: Your language model is secretly a reward model},
  author={Rafailov, Rafael and Sharma, Archit and Mitchell, Eric and Manning, Christopher D and Ermon, Stefano and Finn, Chelsea},
  journal={Advances in neural information processing systems},
  volume={36},
  pages={53728--53741},
  year={2023}
}

@article{guo2025deepseekr1,
  title={Deepseek-r1: Incentivizing reasoning capability in llms via reinforcement learning},
  author={Guo, Daya and Yang, Dejian and Zhang, Haowei and Song, Junxiao and Wang, Peiyi and Zhu, Qihao and Xu, Runxin and Zhang, Ruoyu and Ma, Shirong and Bi, Xiao and others},
  journal={arXiv preprint arXiv:2501.12948},
  year={2025}
}

@inproceedings{sd3,
  title={Scaling rectified flow transformers for high-resolution image synthesis},
  author={Esser, Patrick and Kulal, Sumith and Blattmann, Andreas and Entezari, Rahim and M{\"u}ller, Jonas and Saini, Harry and Levi, Yam and Lorenz, Dominik and Sauer, Axel and Boesel, Frederic and others},
  booktitle={Forty-first international conference on machine learning},
  year={2024}
}

@article{kirstain2023pick,
  title={Pick-a-pic: An open dataset of user preferences for text-to-image generation},
  author={Kirstain, Yuval and Polyak, Adam and Singer, Uriel and Matiana, Shahbuland and Penna, Joe and Levy, Omer},
  journal={Advances in neural information processing systems},
  volume={36},
  pages={36652--36663},
  year={2023}
}

@article{wu2023human,
  title={Human Preference Score v2: A Solid Benchmark for Evaluating Human Preferences of Text-to-Image Synthesis},
  author={Wu, Xiaoshi and Hao, Yiming and Sun, Keqiang and Chen, Yixiong and Zhu, Feng and Zhao, Rui and Li, Hongsheng},
  journal={arXiv preprint arXiv:2306.09341},
  year={2023}
}

@inproceedings{albergo2024stochastic,
  title={Stochastic Interpolants with Data-Dependent Couplings},
  author={Albergo, Michael Samuel and Goldstein, Mark and Boffi, Nicholas Matthew and Ranganath, Rajesh and Vanden-Eijnden, Eric},
  booktitle={International Conference on Machine Learning},
  pages={921--937},
  year={2024},
  organization={PMLR}
}

@misc{sd35,
      title        = {Stable Diffusion 3.5},
      author       = {{Stability AI}},
      year         = {2024},
      howpublished = {\url{https://stability.ai/news/introducing-stable-diffusion-3-5}},
      note         = {Model weights: \url{https://huggingface.co/stabilityai/stable-diffusion-3.5-medium}},
}

@misc{flux2klein,
      title        = {{FLUX.2 [klein]}: Towards Interactive Visual Intelligence},
      author       = {{Black Forest Labs}},
      year         = {2026},
      howpublished = {\url{https://bfl.ai/blog/flux2-klein-towards-interactive-visual-intelligence}},
      note         = {Model weights: \url{https://huggingface.co/black-forest-labs/FLUX.2-klein-base-9B}},
}

@article{geneval2,
  title={GenEval 2: Addressing Benchmark Drift in Text-to-Image Evaluation},
  author={Kamath, Amita and Chang, Kai-Wei and Krishna, Ranjay and Zettlemoyer, Luke and Hu, Yushi and Ghazvininejad, Marjan},
  journal={arXiv preprint arXiv:2512.16853},
  year={2025}
}

@inproceedings{clip,
  title={Learning transferable visual models from natural language supervision},
  author={Radford, Alec and Kim, Jong Wook and Hallacy, Chris and Ramesh, Aditya and Goh, Gabriel and Agarwal, Sandhini and Sastry, Girish and Askell, Amanda and Mishkin, Pamela and Clark, Jack and others},
  booktitle={International conference on machine learning},
  pages={8748--8763},
  year={2021},
  organization={PMLR}
}

@article{mixgrpo,
  title={Mixgrpo: Unlocking flow-based grpo efficiency with mixed ode-sde},
  author={Li, Junzhe and Cui, Yutao and Huang, Tao and Ma, Yinping and Fan, Chun and Cheng, Yiming and Yang, Miles and Zhong, Zhao and Bo, Liefeng},
  journal={arXiv preprint arXiv:2507.21802},
  year={2025}
}

@article{grpoguard,
  title={Grpo-guard: Mitigating implicit over-optimization in flow matching via regulated clipping},
  author={Wang, Jing and Liang, Jiajun and Liu, Jie and Liu, Henglin and Liu, Gongye and Zheng, Jun and Pang, Wanyuan and Ma, Ao and Xie, Zhenyu and Wang, Xintao and others},
  journal={arXiv preprint arXiv:2510.22319},
  year={2025}
}

@article{gdpo,
  title={Gdpo: Group reward-decoupled normalization policy optimization for multi-reward rl optimization},
  author={Liu, Shih-Yang and Dong, Xin and Lu, Ximing and Diao, Shizhe and Belcak, Peter and Liu, Mingjie and Chen, Min-Hung and Yin, Hongxu and Wang, Yu-Chiang Frank and Cheng, Kwang-Ting and others},
  journal={arXiv preprint arXiv:2601.05242},
  year={2026}
}

@inproceedings{
    diffusionnft,
    title={Diffusion{NFT}: Online Diffusion Reinforcement with Forward Process},
    author={Kaiwen Zheng and Huayu Chen and Haotian Ye and Haoxiang Wang and Qinsheng Zhang and Kai Jiang and Hang Su and Stefano Ermon and Jun Zhu and Ming-Yu Liu},
    booktitle={The Fourteenth International Conference on Learning Representations},
    year={2026}
}

@article{awm,
  title={Advantage weighted matching: Aligning rl with pretraining in diffusion models},
  author={Xue, Shuchen and Ge, Chongjian and Zhang, Shilong and Li, Yichen and Ma, Zhi-Ming},
  journal={arXiv preprint arXiv:2509.25050},
  year={2025}
}

@misc{flux1dev,
      title={FLUX.1: Announcing Black Forest Labs},
      author={{Black Forest Labs}},
      year={2024},
      howpublished={\url{https://blackforestlabs.ai/announcing-black-forest-labs/}},
}

\clearpage
\appendix
\appendixtitle
\startcontents[app]
\printcontents[app]{l}{1}{}

\section{The Flow-DPPO Algorithm}
\label{subsec:algorithm}

We summarize the complete Flow-DPPO training procedure. The algorithm adopts the CPS sampling framework \citep{wang2025cps} for trajectory generation, uses group-relative advantage estimation, and applies the divergence-based mask during policy optimization.

\begin{algorithm}[h]
\caption{Flow-DPPO Training}
\label{alg:flow_dppo}
\begin{algorithmic}[1]
\State \textbf{Input:} Flow model $\vv_\theta$, reference model $\vv_\text{ref}$, reward function $R$, prompts $\gC$
\State \textbf{Hyperparameters:} group size $G$, divergence threshold $\delta$, KL coefficient $\beta$, stochasticity $\eta$
\For{each training iteration}
    \State Sample prompts $\{\vc_j\} \sim \gC$
    \State \textbf{// Rollout phase (with $\theta_\text{old}$)}
    \For{each prompt $\vc_j$}
        \State Generate $G$ trajectories $\{(\vx^i_T, \ldots, \vx^i_0)\}_{i=1}^G$ via Flow-SDE(Eq.~(\ref{eq:flow_sde})) or CPS (Eq.~(\ref{eq:cps})) using $\vv_{\theta_\text{old}}$
        \State Record log-probabilities $\log p_{\theta_\text{old}}(\vx^i_{t-\Delta t} \mid \vx^i_t)$ and means $\vmu_{\theta_\text{old}}(\vx^i_t, t)$
        \State Compute rewards $R(\vx^i_0, \vc_j)$ and advantages $\hat{A}^i$ 
    \EndFor
    \State \textbf{// Policy optimization phase}
    \For{each gradient step}
        \State Compute current means $\vmu_\theta(\vx^i_t, t)$ via forward pass of $\vv_\theta$
        \State Compute divergence $D_t = \|\vmu_{\theta_\text{old}}(\vx^i_t, t) - \vmu_\theta(\vx^i_t, t)\|^2$
        \State Compute ratio $r^i_t(\theta)$ from log-probabilities
        \State Compute mask $M^i_t$ (Eq.~(\ref{eq:flow_dppo_mask}))
        \State Update $\theta$ by maximizing $\gL^{\text{Flow-DPPO}}$ (Eq.~(\ref{eq:flow_dppo_obj}))
    \EndFor
    \State $\theta_\text{old} \leftarrow \theta$
\EndFor
\end{algorithmic}
\end{algorithm}

\textbf{Computational overhead.} The divergence computation requires one additional forward pass of the velocity network to obtain $\vmu_\theta(\vx^i_t, t)$ at training time. However, this forward pass is already required for computing the log ratio, so the divergence $D_t = \|\vmu_{\theta_\text{old}} - \vmu_\theta\|^2$ comes at \emph{zero additional cost}: it is simply the squared norm of a difference that is already computed.

\section{Policy Improvement Bound for Flow Models}
\label{app:improvement_bound}

We adapt the classical policy improvement theory~\citep{kakade2002approximately, schulman2015trust} to the finite-horizon, undiscounted setting of flow model denoising, following the approach of \citet{qi2026rethinking} for the LLM regime. We use the MDP notation introduced in \Cref{subsec:RL_finetuning_for_flow_matching_models}: $K-1$ decision steps indexed by $k \in \{1, \dots, K-1\}$, states $\vs_k = (\vc, t_k, \vx_{t_k})$, actions $\va_k = \vx_{t_{k+1}}$, and terminal reward $R(\vx_0, \vc)$.

\subsection{Proof of Performance Difference Identity}
\label{app:pdi_proof}

\begin{proof}[Proof of \Cref{thm:flow_pdi}]
We begin by expressing the performance difference via its definition. Since the reward is only a function of the terminal state $\vx_0$ and the prompt $\vc$, we have:
\begin{align*}
J(\pi_\theta) - J(\pi_{\theta_\mathrm{old}})
&= \E_{\tau \sim \pi_\theta}[R(\vx_0, \vc)] - \E_{\tau \sim \pi_{\theta_\mathrm{old}}}[R(\vx_0, \vc)] \\
&= \int \big(\pi_\theta(\tau \mid \vc) - \pi_{\theta_\mathrm{old}}(\tau \mid \vc)\big) R(\vx_0, \vc)\,\mathrm{d}\tau,
\end{align*}
where the integral is over all trajectories $\tau = (\va_1, \dots, \va_{K-1})$ (we omit the deterministic transition structure for notational clarity).

The core of the proof is the telescoping identity for the difference in trajectory probabilities. Since $\pi_\theta(\tau \mid \vc) = \prod_{k=1}^{K-1} \pi_\theta(\va_k \mid \vs_k)$, we apply the algebraic identity $\prod_{k=1}^N a_k - \prod_{k=1}^N b_k = \sum_{k=1}^N \big(\prod_{j=1}^{k-1} b_j\big)(a_k - b_k)\big(\prod_{j=k+1}^N a_j\big)$:
\begin{align*}
\pi_\theta(\tau \mid \vc) - \pi_{\theta_\mathrm{old}}(\tau \mid \vc)
= \sum_{k=1}^{K-1}
\left(\prod_{j=1}^{k-1} \pi_{\theta_\mathrm{old}}(\va_j \mid \vs_j)\right) 
\cdot
\big(\pi_\theta(\va_k \mid \vs_k) - \pi_{\theta_\mathrm{old}}(\va_k \mid \vs_k)\big)
\left(\prod_{j=k+1}^{K-1} \pi_\theta(\va_j \mid \vs_j)\right).
\end{align*}
Substituting into the performance difference and converting to an expectation under $\pi_{\theta_\mathrm{old}}$:
\begin{align*}
J(\pi_\theta) - J(\pi_{\theta_\mathrm{old}})
&= \E_{\tau \sim \pi_{\theta_\mathrm{old}}} \left[ R(\vx_0, \vc) \sum_{k=1}^{K-1} \left(\frac{\pi_\theta(\va_k \mid \vs_k)}{\pi_{\theta_\mathrm{old}}(\va_k \mid \vs_k)} - 1\right) \left(\prod_{j=k+1}^{K-1} \frac{\pi_\theta(\va_j \mid \vs_j)}{\pi_{\theta_\mathrm{old}}(\va_j \mid \vs_j)}\right) \right].
\end{align*}
We decompose this expression by adding and subtracting the term where the future ratio product is set to 1:
\begin{align*}
J(\pi_\theta) - J(\pi_{\theta_\mathrm{old}})
&= \underbrace{\E_{\tau \sim \pi_{\theta_\mathrm{old}}} \left[ R(\vx_0, \vc) \sum_{k=1}^{K-1} \left(\frac{\pi_\theta(\va_k \mid \vs_k)}{\pi_{\theta_\mathrm{old}}(\va_k \mid \vs_k)} - 1\right) \right]}_{L'_{\theta_\mathrm{old}}(\pi_\theta)} \\
&\quad - \underbrace{\E_{\tau \sim \pi_{\theta_\mathrm{old}}} \left[ R(\vx_0, \vc) \sum_{k=1}^{K-1} \left(\frac{\pi_\theta(\va_k \mid \vs_k)}{\pi_{\theta_\mathrm{old}}(\va_k \mid \vs_k)} - 1\right) \left(1 - \prod_{j=k+1}^{K-1} \frac{\pi_\theta(\va_j \mid \vs_j)}{\pi_{\theta_\mathrm{old}}(\va_j \mid \vs_j)}\right) \right]}_{\Delta(\pi_{\theta_\mathrm{old}}, \pi_\theta)}.
\end{align*}
This completes the proof.
\end{proof}

\subsection{Proof of Policy Improvement Bound}
\label{app:improvement_bound_proof}

\begin{lemma}[Bound on Trajectory-Level TV Divergence]
\label{lem:flow_trajectory_tv}
Let $\pi_{\theta_\mathrm{old}}$ and $\pi_\theta$ be two policies for the flow model MDP. Let $\pi_{\theta_\mathrm{old},>k}(\cdot \mid \vs_{k+1})$ and $\pi_{\theta,>k}(\cdot \mid \vs_{k+1})$ denote the distributions over future sub-trajectories $(\va_{k+1}, \dots, \va_{K-1})$ starting from state $\vs_{k+1}$. Then:
\begin{equation*}
\TV\big(\pi_{\theta_\mathrm{old},>k}(\cdot \mid \vs_{k+1}) \| \pi_{\theta,>k}(\cdot \mid \vs_{k+1})\big) \le \sum_{j=k+1}^{K-1} \E_{\vs_j \sim \pi_{\theta_\mathrm{old}}} \left[ \TV\big(\pi_{\theta_\mathrm{old}}(\cdot \mid \vs_j) \| \pi_\theta(\cdot \mid \vs_j)\big) \right],
\end{equation*}
where the expectation is over states visited under $\pi_{\theta_\mathrm{old}}$ starting from $\vs_{k+1}$.
\end{lemma}

\begin{proof}
Let $P(\tau_{>k}) = \pi_{\theta_\mathrm{old},>k}(\tau_{>k} \mid \vs_{k+1})$ and $Q(\tau_{>k}) = \pi_{\theta,>k}(\tau_{>k} \mid \vs_{k+1})$, where $\tau_{>k} = (\va_{k+1}, \dots, \va_{K-1})$. We have:
\begin{equation*}
2\TV(P \| Q) = \int |P(\tau_{>k}) - Q(\tau_{>k})|\,\mathrm{d}\tau_{>k} = \int \left|\prod_{j=k+1}^{K-1} \pi_{\theta_\mathrm{old}}(\va_j \mid \vs_j) - \prod_{j=k+1}^{K-1} \pi_\theta(\va_j \mid \vs_j)\right| \mathrm{d}\tau_{>k}.
\end{equation*}
Applying the telescoping identity $|a_1 \cdots a_N - b_1 \cdots b_N| \le \sum_{j=1}^N \big(\prod_{i=1}^{j-1} a_i\big) |a_j - b_j| \big(\prod_{i=j+1}^N b_i\big)$ (which follows from the triangle inequality) and integrating:
\begin{align*}
2 \TV(P \| Q)
\le \sum_{j=k+1}^{K-1} \int
\left(\prod_{i=k+1}^{j-1} \pi_{\theta_\mathrm{old}}(\va_i \mid \vs_i)\right)
\left|\pi_{\theta_\mathrm{old}}(\va_j \mid \vs_j) - \pi_\theta(\va_j \mid \vs_j)\right| \left(\prod_{i=j+1}^{K-1} \pi_\theta(\va_i \mid \vs_i)\right) \mathrm{d}\tau_{>k}.
\end{align*}
For each term indexed by $j$, integrating out the future actions $\va_{j+1}, \dots, \va_{K-1}$ yields 1 (since $\pi_\theta$ is normalized), leaving:
\begin{align*}
2 \TV(P \| Q)
\le \sum_{j=k+1}^{K-1} \int 
\left(\prod_{i=k+1}^{j-1} \pi_{\theta_\mathrm{old}}(\va_i \mid \vs_i)\right)
\left(\int |\pi_{\theta_\mathrm{old}}(\va_j \mid \vs_j) - \pi_\theta(\va_j \mid \vs_j)|\,\mathrm{d}\va_j\right) \mathrm{d}\va_{k+1} \cdots \mathrm{d}\va_{j-1}.
\end{align*}
The inner integral is $2 \TV (\pi_{\theta_\mathrm{old}}(\cdot \mid \vs_j) \| \pi_\theta(\cdot \mid \vs_j))$, and the outer integral defines an expectation over states $\vs_j$ under policy $\pi_{\theta_\mathrm{old}}$. Thus:
\begin{equation*}
\TV(P \| Q) \le \sum_{j=k+1}^{K-1} \E_{\vs_j \sim \pi_{\theta_\mathrm{old}}} \left[ \TV\big(\pi_{\theta_\mathrm{old}}(\cdot \mid \vs_j) \| \pi_\theta(\cdot \mid \vs_j)\big) \right]. \qedhere
\end{equation*}
\end{proof}

\begin{proof}[Proof of \Cref{thm:flow_improvement_bound}]
From \Cref{thm:flow_pdi}, we start with the exact performance difference identity:
\begin{equation*}
J(\pi_\theta) - J(\pi_{\theta_\mathrm{old}}) = L'_{\theta_\mathrm{old}}(\pi_\theta) - \Delta(\pi_{\theta_\mathrm{old}}, \pi_\theta).
\end{equation*}
Our goal is to upper-bound $|\Delta(\pi_{\theta_\mathrm{old}}, \pi_\theta)|$. We begin by bounding the reward by its maximum absolute value $\xi = \max_{\vx_0, \vc} |R(\vx_0, \vc)|$:
\begin{align}
\label{eq:app_delta_intermediate}
\begin{split}
&|\Delta(\pi_{\theta_\mathrm{old}}, \pi_\theta)|
\\
&\le \xi \cdot \E_{\tau \sim \pi_{\theta_\mathrm{old}}} \left[ \sum_{k=1}^{K-1} \left|\frac{\pi_\theta(\va_k \mid \vs_k)}{\pi_{\theta_\mathrm{old}}(\va_k \mid \vs_k)} - 1\right| \cdot \left|1 - \prod_{j=k+1}^{K-1} \frac{\pi_\theta(\va_j \mid \vs_j)}{\pi_{\theta_\mathrm{old}}(\va_j \mid \vs_j)}\right| \right] \\
&= \xi \cdot \sum_{k=1}^{K-1} \E_{\vs_{\le k} \sim \pi_{\theta_\mathrm{old}}} \left[ \left|\frac{\pi_\theta(\va_k \mid \vs_k)}{\pi_{\theta_\mathrm{old}}(\va_k \mid \vs_k)} - 1\right| \cdot \E_{\tau_{>k} \sim \pi_{\theta_\mathrm{old}}} \left[ \left|1 - \frac{\pi_{\theta,>k}(\tau_{>k} \mid \vs_{k+1})}{\pi_{\theta_\mathrm{old},>k}(\tau_{>k} \mid \vs_{k+1})}\right| \right] \right].
\end{split}
\end{align}
The inner expectation over future sub-trajectories is exactly twice the TV divergence between future trajectory distributions:
\begin{equation*}
\E_{\tau_{>k} \sim \pi_{\theta_\mathrm{old}}} \left[\left|1 - \frac{\pi_{\theta,>k}(\tau_{>k} \mid \vs_{k+1})}{\pi_{\theta_\mathrm{old},>k}(\tau_{>k} \mid \vs_{k+1})}\right|\right] =
2\TV\big(\pi_{\theta_\mathrm{old},>k}(\cdot \mid \vs_{k+1}) \| \pi_{\theta,>k}(\cdot \mid \vs_{k+1})\big).
\end{equation*}
Applying \Cref{lem:flow_trajectory_tv} and bounding each term by $D_{\mathrm{TV}}^{\max}$:
\begin{equation*}
\TV\big(\pi_{\theta_\mathrm{old},>k}(\cdot \mid \vs_{k+1}) \| \pi_{\theta,>k}(\cdot \mid \vs_{k+1})\big) \le (K-1-k)\, D_{\mathrm{TV}}^{\max}(\pi_{\theta_\mathrm{old}} \| \pi_\theta).
\end{equation*}
Substituting back into Eq.~(\ref{eq:app_delta_intermediate}):
\begin{align*}
|\Delta(\pi_{\theta_\mathrm{old}}, \pi_\theta)|
&\le \xi \cdot \sum_{k=1}^{K-1} \E_{\vs_k \sim \pi_{\theta_\mathrm{old}}} \left[ \E_{\va_k \sim \pi_{\theta_\mathrm{old}}(\cdot|\vs_k)} \left[\left|\frac{\pi_\theta(\va_k \mid \vs_k)}{\pi_{\theta_\mathrm{old}}(\va_k \mid \vs_k)} - 1\right|\right] \right] \cdot 2(K-1-k)\, D_{\mathrm{TV}}^{\max} \\
&= 2\xi \cdot D_{\mathrm{TV}}^{\max} \sum_{k=1}^{K-1} (K-1-k) \cdot \E_{\vs_k \sim \pi_{\theta_\mathrm{old}}} \left[ 2D_{\mathrm{TV}}\big(\pi_{\theta_\mathrm{old}}(\cdot \mid \vs_k) \| \pi_\theta(\cdot \mid \vs_k)\big) \right] \\
&\le 2\xi \cdot D_{\mathrm{TV}}^{\max} \sum_{k=1}^{K-1} (K-1-k) \cdot 2D_{\mathrm{TV}}^{\max} \\
&= 4\xi \cdot {D_{\mathrm{TV}}^{\max}}^2 \sum_{k=1}^{K-1} (K-1-k).
\end{align*}
Evaluating the sum: $\sum_{k=1}^{K-1}(K-1-k) = \sum_{m=0}^{K-2} m = \frac{(K-1)(K-2)}{2}$. Therefore:
\begin{equation*}
|\Delta(\pi_{\theta_\mathrm{old}}, \pi_\theta)| \le 4\xi \cdot \frac{(K-1)(K-2)}{2} \cdot {D_{\mathrm{TV}}^{\max}}^2 = 2\xi(K-1)(K-2) \cdot {D_{\mathrm{TV}}^{\max}}^2.
\end{equation*}
Substituting into the performance difference identity yields the desired bound:
\begin{equation*}
J(\pi_\theta) - J(\pi_{\theta_\mathrm{old}}) \ge L'_{\theta_\mathrm{old}}(\pi_\theta) - 2\xi (K-1)(K-2) \cdot {D_{\mathrm{TV}}^{\max}(\pi_{\theta_\mathrm{old}} \| \pi_\theta)}^2.
\end{equation*}
This completes the proof.
\end{proof}

\subsection{A Tighter Policy Improvement Bound}
\label{app:tighter_linear_bound}

The quadratic dependence on the horizon $K^2$ in \Cref{thm:flow_improvement_bound} can be overly pessimistic. By exploiting the fact that $D_{\mathrm{TV}} \le 1$, we derive a tighter bound that is linear in $K$.

Starting from the intermediate step in Eq.~(\ref{eq:app_delta_intermediate}), the inner expectation is $2 \TV(\pi_{\theta_\mathrm{old},>k}(\cdot \mid \vs_{k+1}) \| \pi_{\theta,>k}(\cdot \mid \vs_{k+1}))$.
Instead of applying \Cref{lem:flow_trajectory_tv}, we directly use the universal bound $D_{\mathrm{TV}} \le 1$:
\begin{align*}
|\Delta(\pi_{\theta_\mathrm{old}}, \pi_\theta)|
&\le \xi \cdot \sum_{k=1}^{K-1} \E_{\vs_k \sim \pi_{\theta_\mathrm{old}}} \left[\E_{\va_k \sim \pi_{\theta_\mathrm{old}}(\cdot|\vs_k)} \left[\left|\frac{\pi_\theta(\va_k \mid \vs_k)}{\pi_{\theta_\mathrm{old}}(\va_k \mid \vs_k)} - 1\right|\right]\right] \cdot 2 \\
&= 4\xi \cdot \E_{\tau \sim \pi_{\theta_\mathrm{old}}} \left[\sum_{k=1}^{K-1} D_{\mathrm{TV}}\big(\pi_{\theta_\mathrm{old}}(\cdot \mid \vs_k) \| \pi_\theta(\cdot \mid \vs_k)\big)\right].
\end{align*}

Combining both bounds, the policy improvement satisfies the composite guarantee:
\begin{equation*}
J(\pi_\theta) - J(\pi_{\theta_\mathrm{old}}) \ge L'_{\theta_\mathrm{old}}(\pi_\theta) - \min\left(2\xi (K\!-\!1)(K\!-\!2) \cdot {D_{\mathrm{TV}}^{\max}}^2,\; 4\xi \cdot \E_{\tau \sim \pi_{\theta_\mathrm{old}}} \left[\sum_{k=1}^{K-1} D_{\mathrm{TV},k}\right]\right),
\end{equation*}
where $D_{\mathrm{TV},k} = D_{\mathrm{TV}}(\pi_{\theta_\mathrm{old}}(\cdot \mid \vs_k) \| \pi_\theta(\cdot \mid \vs_k))$. The quadratic bound is tighter for small policy changes, while the linear bound is tighter for larger updates or longer horizons.

\subsection{Connection to Gaussian Per-Step Divergence}
\label{app:gaussian_tv}

For the Gaussian policies in Eq.~(\ref{eq:abstract_policy}), $\pi_{\theta_\mathrm{old}}(\cdot \mid \vs_k) = \gN(\vmu_{\theta_\mathrm{old}}, \sigma^2(t_k)\rmI)$ and $\pi_\theta(\cdot \mid \vs_k) = \gN(\vmu_\theta, \sigma^2(t_k)\rmI)$, the TV divergence admits the closed form:
\begin{equation*}
D_{\mathrm{TV}}\big(\pi_{\theta_\mathrm{old}}(\cdot \mid \vs_k) \| \pi_\theta(\cdot \mid \vs_k)\big) = 2\Phi\!\left(\frac{\|\vmu_{\theta_\mathrm{old}} - \vmu_\theta\|}{2\sigma(t_k)}\right) - 1,
\end{equation*}
where $\Phi$ is the standard normal CDF. Since $\Phi$ is strictly monotonically increasing, the TV constraint $D_{\mathrm{TV}}^{\max} \le \delta$ is equivalent to:
\begin{equation*}
\max_{\vs_k}\; \|\vmu_{\theta_\mathrm{old}}(\vx_{t_k}, t_k, \vc) - \vmu_\theta(\vx_{t_k}, t_k, \vc)\|^2 \le 4\sigma^2(t_k) \left[\Phi^{-1}\!\left(\frac{1+\delta}{2}\right)\right]^2 \eqqcolon \delta'.
\end{equation*}
This formally establishes that the Flow-DPPO mask, which blocks updates when $\|\vmu_{\theta_\mathrm{old}} - \vmu_\theta\|^2 > \delta$, implements a trust-region constraint equivalent (up to a monotone rescaling) to constraining the per-step TV divergence. The policy improvement bound (\Cref{thm:flow_improvement_bound}) thus provides a rigorous theoretical guarantee for Flow-DPPO: by enforcing a per-step divergence threshold, the penalty term remains controlled, ensuring monotonic policy improvement.

\section{KL Divergence Between Gaussian Policies}
\label{app:kl_derivation}

In this section, we derive the KL divergence between old and new policies in flow models and establish its connection to the TV divergence used in the policy improvement bound.

\subsection{General Gaussian KL Divergence}
\label{app:gaussian_kl}

Let $p = \gN(\vmu_1, \sigma^2 \rmI)$ and $q = \gN(\vmu_2, \sigma^2 \rmI)$ be two isotropic Gaussians in $\R^d$ with the same covariance. The KL divergence is:
\begin{align}
\KL(p \| q) &= \frac{1}{2\sigma^2} \left[ 2(\vmu_1 - \vmu_2)^\top \underbrace{\E_{p}[\vx - \vmu_1]}_{= \bm{0}} + \|\vmu_1 - \vmu_2\|^2 \right]
= \frac{\|\vmu_1 - \vmu_2\|^2}{2\sigma^2}.
\label{eq:app_kl_gaussian}
\end{align}
Note that this is symmetric in the means: $\KL(p \| q) = \KL(q \| p)$ when the covariances are identical.

\subsection{Connection Between KL and TV in the Gaussian Setting}
\label{app:tv_connection}

For the same pair of Gaussians, the TV divergence is:
\begin{equation*}
D_{\mathrm{TV}}(p, q) = 2\Phi\!\left(\frac{\|\vmu_1 - \vmu_2\|}{2\sigma}\right) - 1.
\end{equation*}
Since both KL and TV are monotone functions of the single quantity $\|\vmu_1 - \vmu_2\|/\sigma$, thresholding one is equivalent to thresholding the other. Specifically, the constraint $D_{\mathrm{TV}} \le \delta_{\mathrm{TV}}$ is equivalent to $\|\vmu_1 - \vmu_2\|^2 \le 4\sigma^2 [\Phi^{-1}((1+\delta_{\mathrm{TV}})/2)]^2$, which in turn is equivalent to $\KL \le 2[\Phi^{-1}((1+\delta_{\mathrm{TV}})/2)]^2$. This shows that the squared $\ell_2$ distance $\|\vmu_{\theta_\mathrm{old}} - \vmu_\theta\|^2$ used in our mask is a unified divergence measure equivalent (up to monotone transformations) to both KL and TV divergences.

\subsection{Application to Flow-SDE}
\label{app:kl_flow_sde}

For Flow-SDE (Eq.~(\ref{eq:flow_sde})), the per-step policy is $\pi_\theta(\vx_{t-\Delta t} \mid \vx_t) = \gN(\vmu_\theta, \sigma_t^2\Delta t \cdot \rmI)$ where:
\begin{equation*}
\vmu_\theta(\vx_t, t) = \vx_t + \left[\vv_\theta(\vx_t, t) + \frac{\sigma_t^2}{2t}\big(\vx_t + (1-t)\vv_\theta(\vx_t, t)\big)\right]\Delta t.
\end{equation*}
The difference in means is:
\begin{equation*}
\vmu_\theta - \vmu_{\theta_\text{old}} = \left(1 + \frac{\sigma_t^2(1-t)}{2t}\right)\Delta t \cdot (\vv_\theta - \vv_{\theta_\text{old}}).
\end{equation*}
Substituting into Eq.~(\ref{eq:app_kl_gaussian}) with $\sigma^2 = \sigma_t^2 \Delta t$:
\begin{equation}
\label{eq:app_kl_sde}
\KL^{\text{SDE}}(\pi_{\theta_\text{old}} \| \pi_\theta) = \frac{\Delta t}{2}\left(\frac{1}{\sigma_t} + \frac{\sigma_t(1-t)}{2t}\right)^2 \|\vv_\theta(\vx_t, t) - \vv_{\theta_\text{old}}(\vx_t, t)\|^2.
\end{equation}

\subsection{Application to CPS}
\label{app:kl_cps}

For CPS (Eq.~(\ref{eq:cps})), the policy mean is $\vmu_\theta^{\text{CPS}} = (1-(t-\Delta t))\hat{\vx}_0 + (t-\Delta t)\cos(\eta\pi/2)\hat{\vx}_1$ and the variance is $\sigma_{\text{CPS}}^2 = (t-\Delta t)^2\sin^2(\eta\pi/2)$. Using $\hat{\vx}_0 = \vx_t - t\,\vv_\theta$ and $\hat{\vx}_1 = \vx_t + (1-t)\vv_\theta$, the difference in means is:
\begin{equation*}
\vmu_\theta^{\text{CPS}} - \vmu_{\theta_\text{old}}^{\text{CPS}} = \left[-(1-(t-\Delta t))t + (t-\Delta t)(1-t)\cos(\eta\pi/2)\right](\vv_\theta - \vv_{\theta_\text{old}}).
\end{equation*}
Let $c(t) = -(1-(t-\Delta t))t + (t-\Delta t)(1-t)\cos(\eta\pi/2)$. Then:
\begin{equation}
\label{eq:app_kl_cps}
\KL^{\text{CPS}}(\pi_{\theta_\text{old}} \| \pi_\theta) = \frac{c(t)^2 \|\vv_\theta - \vv_{\theta_\text{old}}\|^2}{2(t-\Delta t)^2\sin^2(\eta\pi/2)}.
\end{equation}
In previous work~\citep{wang2025cps},
the $2\sigma_{\text{CPS}}^2$ normalization is dropped for numerical stability,
reducing the divergence to $D(\pi_{\theta_\text{old}} \| \pi_\theta) = \|\vmu_{\theta_\text{old}}^{\text{CPS}} - \vmu_\theta^{\text{CPS}}\|^2$.
We instead retain the full normalization in Eq.~(\ref{eq:app_kl_cps}):
because $\sigma_{\text{CPS}}^2 \propto (t-\Delta t)^2$ shrinks at later denoising steps,
the $\sigma_{\text{CPS}}^{-2}$ factor amplifies the divergence where small velocity changes
most affect the output, yielding a tighter constraint that prevents distribution collapse.

\section{Ratio Variance Analysis}
\label{app:ratio_variance}

We provide a detailed analysis of the variance of the log-ratio in flow models.

From Eq.~(\ref{eq:ratio_decomposition}), $\log r^i_t = {\epsilonv^\top \vd}/{\sigma} - {\|\vd\|^2}/{(2\sigma^2)}$, where $\vd = \vmu_\theta - \vmu_{\theta_\text{old}}$ and $\epsilonv \sim \gN(\bm{0}, \rmI)$. It follows that:
\begin{equation*}
\E[\log r^i_t] = -\frac{\|\vd\|^2}{2\sigma^2} = -\KL(\pi_{\theta_\text{old}} \| \pi_\theta), \qquad
\Var[\log r^i_t] = \frac{\|\vd\|^2}{\sigma^2} = 2\,\KL(\pi_{\theta_\text{old}} \| \pi_\theta).
\end{equation*}

Thus $\mathrm{std}[\log r^i_t] = \sqrt{2\,\KL}$. When the KL is moderate (e.g., $\KL = 0.5$), the standard deviation of the log-ratio is $1.0$, meaning that individual log-ratio samples fluctuate by $\pm 1$ around the mean of $-0.5$. In terms of the ratio itself, this corresponds to roughly a $3\times$ multiplicative spread.

\textbf{Implication for clipping.} With a typical clip parameter $\epsilon = 0.2$ (i.e., clip range $[0.8, 1.2]$), the log-clip range is $[\log 0.8, \log 1.2] \approx [-0.22, 0.18]$. Comparing this narrow range with the log-ratio standard deviation of $\sqrt{2\,\KL}$, we see that even for modest KL values, a significant fraction of samples will be clipped purely due to noise, not because the true divergence is excessive. This provides rigorous justification for replacing ratio-based clipping with direct divergence measurement.

\section{Towards a Predictive Divergence Mask}
\label{app:future_work}

We recall the asymmetric mask in Flow-DPPO (Eq.~(\ref{eq:flow_dppo_mask})). The mask blocks the gradient (i.e., $M^i_t = 0$) when two conditions hold simultaneously: (i)~the divergence $D_t > \delta$ already exceeds the trust-region threshold, and (ii)~a directional condition signals that the optimization would push the policy \emph{further away} from $\pi_{\theta_\mathrm{old}}$. Concretely, the directional condition triggers when $\hat{A}^i > 0 \wedge r^i_t > 1$ (the gradient would further increase an already-elevated ratio) or $\hat{A}^i < 0 \wedge r^i_t < 1$ (the gradient would further decrease an already-reduced ratio). These two cases can be compactly unified as:
\begin{equation}
\label{eq:mask_sign_form}
M^i_t = 0 \quad\Longleftrightarrow\quad \operatorname{sgn}\!\big(\hat{A}^i \cdot (r^i_t - 1)\big) > 0 \;\;\wedge\;\; D_t > \delta.
\end{equation}

While this design is effective in practice, the directional indicator $\operatorname{sgn}\!\big(\hat{A}^i (r^i_t - 1)\big)$ is a \emph{heuristic proxy} for whether the upcoming gradient step will increase the divergence. In a ratio-based trust region (e.g., PPO clipping), this sign test is well-motivated: $r^i_t - 1$ directly reflects the deviation of the single-sample Monte Carlo estimate of the importance ratio, so the sign of $\hat{A}^i(r^i_t - 1)$ faithfully indicates whether the surrogate objective would drive the ratio further from unity. However, in a \emph{divergence-based} trust region where the constraint is on $D_t = \KL(\pi_{\theta_\mathrm{old}} \| \pi_\theta)$, the connection is less direct. The ratio $r^i_t$ is a stochastic quantity evaluated at a single sampled action, whereas $D_t$ measures a distributional distance that integrates over all actions. A positive $\hat{A}^i(r^i_t - 1)$ does not guarantee that the gradient step will increase $D_t$, nor does a negative value guarantee a decrease.

In this section we exploit the Gaussian structure of flow model policies to derive a more principled masking criterion. We first predict how a single gradient step changes $D_t$ (\S\ref{app:predict_divergence}), obtaining a closed-form expression that decomposes into a first-order directional term and a second-order magnitude term. The sign of the first-order term yields an exact directional criterion $\operatorname{sgn}\!\big(\hat{A} \cdot (\log r_t - D_t)\big)$, which recovers the current sign test in the small-divergence regime but reveals a correction when the policy has already drifted. The full expression further accounts for the step size and gradient magnitude, leading to a predictive mask (\S\ref{app:predictive_mask}) that directly forecasts whether the post-update divergence will exceed $\delta$.

\subsection{Predicting Post-Update Divergence}
\label{app:predict_divergence}

Fix a denoising step with state $\vx_t$ and suppress the time index for brevity. Write $\vmu \equiv \vmu_\theta(\vx_t, t)$, $\vmu_\mathrm{old} \equiv \vmu_{\theta_\mathrm{old}}(\vx_t, t)$, $\vd = \vmu - \vmu_\mathrm{old}$, and $D_t = \|\vd\|^2/(2\sigma^2)$. The sampled action is $\vx_{t-\Delta t} = \vmu_\mathrm{old} + \sigma\epsilonv$ with $\epsilonv \sim \gN(\bm{0}, \rmI)$.

We derive how a single gradient step on the surrogate objective $L = r_t \cdot \hat{A}$ changes the divergence $D_t$. The policy gradient with respect to $\vmu$ is:
\begin{equation*}
\nabla_{\vmu} L = \hat{A} \cdot r_t \cdot \nabla_{\vmu} \log r_t = \frac{\hat{A} \cdot r_t}{\sigma^2}(\sigma\epsilonv - \vd).
\end{equation*}
With effective learning rate $\eta$, the updated mean is $\vmu_\mathrm{new} = \vmu + \eta \cdot \nabla_{\vmu} L$. Let $\vg = \sigma\epsilonv - \vd$. The predicted post-update divergence is:
\begin{align}
D_t^\mathrm{new} &= \frac{\|\vmu_\mathrm{new} - \vmu_\mathrm{old}\|^2}{2\sigma^2} = \frac{1}{2\sigma^2}\left\|\vd + \frac{\eta\,\hat{A}\,r_t}{\sigma^2}\,\vg\right\|^2 \nonumber\\
&= D_t + \frac{\eta\,\hat{A}\,r_t}{\sigma^4}\,\vg^\top\vd + \frac{\eta^2\hat{A}^2 r_t^2}{2\sigma^6}\|\vg\|^2.
\label{eq:predicted_divergence}
\end{align}
From the ratio decomposition (Eq.~(\ref{eq:ratio_decomposition})), $\log r_t = \epsilonv^\top\vd / \sigma - \|\vd\|^2/(2\sigma^2)$, which gives $\vg^\top\vd = \sigma^2(\log r_t - D_t)$. The first-order term thus simplifies to $(\eta\,\hat{A}\,r_t / \sigma^2)(\log r_t - D_t)$. The three terms in Eq.~(\ref{eq:predicted_divergence}) have clear interpretations: (1)~the current divergence $D_t$; (2)~a first-order term whose sign determines whether the gradient step increases or decreases the divergence; (3)~a non-negative second-order term that grows with the step size $\eta$ and gradient magnitude $\|\vg\|$, always contributing positively to $D_t^\mathrm{new}$.

\textbf{The first-order directional criterion.}
The direction of divergence change is mainly determined by the sign of the first-order term. Since $r_t > 0$ and $\eta > 0$, this sign equals:
\begin{equation}
\label{eq:directional_criterion}
\operatorname{sgn}\!\Big(\hat{A} \cdot \big(\log r_t - D_t\big)\Big).
\end{equation}
When this is positive, the gradient step increases $D_t$; when negative, it decreases $D_t$. Equivalently, this is the sign of the inner product $\langle \nabla_\vmu L,\, \nabla_\vmu D_t \rangle$, confirming that the surrogate gradient projects onto the divergence-increasing direction.

\textbf{Recovery of the current mask.} In the small-divergence regime $D_t \ll 1$ (which is the typical operating range when the trust region is effective), the correction $D_t \approx 0$ and the criterion simplifies to $\operatorname{sgn}(\hat{A} \cdot \log r_t)$. Since $\operatorname{sgn}(\log r_t) = \operatorname{sgn}(r_t - 1)$, this is equivalent to $\operatorname{sgn}\!\big(\hat{A} \cdot (r_t - 1)\big)$, which is exactly the directional condition in Eq.~(\ref{eq:mask_sign_form}). Thus, the current Flow-DPPO mask implements the correct first-order divergence-increasing criterion in this regime.

\textbf{The correction term.} When $D_t$ is non-negligible (i.e., the policy has already drifted appreciably), the true divergence-change direction is $\operatorname{sgn}\!\big(\hat{A} \cdot (\log r_t - D_t)\big)$ rather than $\operatorname{sgn}\!\big(\hat{A} \cdot (r_t - 1)\big)$. The subtracted term $D_t$ shifts the decision boundary: a sample must have $\log r_t > D_t > 0$ (rather than merely $\log r_t > 0$) before the positive-advantage gradient is classified as divergence-increasing. Intuitively, when the policy has already moved away from $\pi_{\theta_\mathrm{old}}$, a moderately elevated ratio does not necessarily push it further; only sufficiently large ratios do. This yields a first natural refinement of the mask: replacing $\operatorname{sgn}\!\big(\hat{A}(r_t - 1)\big)$ with $\operatorname{sgn}\!\big(\hat{A} \cdot (\log r_t - D_t)\big)$ as the directional indicator, which we call the \emph{first-order predictive mask}:
\begin{equation}
\label{eq:first_order_mask}
M_t^{(1)} =
\begin{cases}
0, & \text{if } \operatorname{sgn}\!\Big(\hat{A} \cdot \big(\log r_t - D_t\big)\Big) > 0 \;\wedge\; D_t > \delta, \\
1, & \text{otherwise}.
\end{cases}
\end{equation}
This mask uses only quantities already computed during training ($\hat{A}$, $r_t$, $D_t$) and requires no additional hyperparameters beyond the existing threshold $\delta$.

\subsection{The Predictive Mask}
\label{app:predictive_mask}

Based on Eq.~(\ref{eq:predicted_divergence}), we define the \emph{(full) predictive mask} that blocks updates whenever the predicted post-update divergence would exceed $\delta$:
\begin{equation}
\label{eq:predictive_mask}
M_t^\mathrm{pred} =
\begin{cases}
0, & \text{if } D_t^\mathrm{new} > \delta, \\
1, & \text{otherwise}.
\end{cases}
\end{equation}

\textbf{Comparison with the first-order mask.} The first-order mask (Eq.~(\ref{eq:first_order_mask})) only considers the direction of divergence change and still relies on the separate threshold condition $D_t > \delta$. The full predictive mask unifies both into a single inequality: whether the gradient increases or decreases divergence is automatically encoded in the predicted value $D_t^\mathrm{new}$, and the threshold comparison is applied to the predicted (rather than current) divergence. This has two consequences. First, when $D_t \ll \delta$, even a divergence-increasing step may be permitted if the predicted $D_t^\mathrm{new}$ remains below $\delta$. Second, when $D_t$ is close to $\delta$, the second-order term $\eta^2\|\vg\|^2$ may push $D_t^\mathrm{new}$ above $\delta$ even when the first-order direction is ``safe'' (i.e., the first-order mask would not fire), correctly blocking large gradient steps near the trust-region boundary.

\textbf{Recovery of the existing mask.} In the limit $\eta \to 0$, the second-order term vanishes and $D_t^\mathrm{new} > \delta$ reduces to requiring that the first-order direction is positive and $D_t > \delta$. Combined with the small-divergence approximation ($D_t \approx 0$), this exactly recovers the current Flow-DPPO mask (Eq.~(\ref{eq:mask_sign_form})).

\subsection{Discussion on Mask Variants}
\label{app:predictive_mask_discussion}

\textbf{Hierarchy of masks.} The three masks form a natural hierarchy of increasing fidelity:
\begin{equation*}
\underbrace{\operatorname{sgn}\!\big(\hat{A}(r_t\!-\!1)\big)}_{\text{current (Eq.~(\ref{eq:mask_sign_form}))}}
\;\;\subset\;\;
\underbrace{\operatorname{sgn}\!\big(\hat{A}(\log r_t\!-\!D_t)\big)}_{\text{first-order (Eq.~(\ref{eq:first_order_mask}))}}
\;\;\subset\;\;
\underbrace{D_t^\mathrm{new} > \delta}_{\text{full predictive (Eq.~(\ref{eq:predictive_mask}))}}.
\end{equation*}
The current mask is the cheapest (no additional computation) and suffices when the trust region keeps $D_t$ small throughout training. The first-order mask refines the directional decision with zero additional hyperparameters. The full predictive mask additionally requires an effective learning rate estimate but provides quantitative divergence prediction.

\textbf{Local approximation.} The analysis treats $\vmu$ as a free vector, whereas in practice it is the output of a neural network. The actual change in $\vmu(\vx_t, t)$ is coupled to changes at all other inputs through shared parameters. The predictive mask is thus a local approximation that is most accurate when the effective learning rate is small and the network Jacobian is approximately preserved across one step.

We leave empirical validation of the predictive masks to future work. The key contribution of this analysis is twofold: it provides a theoretical justification for the existing asymmetric condition (showing it is the correct first-order criterion in the small-divergence regime), and it charts a principled path toward more refined trust-region enforcement that exploits the Gaussian structure of flow model policies.

\section{Experimental Details}
\label{sec:appendix_experimental_details}

\subsection{Computational Resources.}
\label{sec:appendix_compute}
All experiments are conducted on NVIDIA H20 96GB GPUs.
The main results in Table~\ref{tab:main_results} require approximately 90K GPU hours in total (across SD3.5, FLUX2-klein-base-9B, and FLUX1-dev with all methods and reward configurations).
Including all ablation studies, multi-epoch experiments, and auxiliary runs,
the overall computational cost for all experiments reported in this paper is approximately 140K GPU hours.

\subsection{Hyperparameters.}
\label{sec:appendix_hyperparams}
LoRA is used for all models. We use LoRA $r=32$ and $\alpha=64$ for SD3.5, $r=64$ and $\alpha=128$ for FLUX2-9B and FLUX.1-dev.
The learning rate is set to \(3 \times 10^{-4}\) for all models aligning to previous works.
We set the training resolution to \(512 \times 512\), number of denoising steps to 10 for SD3.5 and 14 for FLUX2-9B.

For GRPO setting, we use group size 16 and number of groups 64 per epoch for all methods.
The PPO clip threshold is set to \(1 \times 10^{-4}\) for Flow-GRPO and Flow-CPS, and \(4 \times 10^{-6}\) for GRPO-Guard, following the official recommendation.
The thresholds for KL-clipping are set to \(1 \times 10^{-7}\) for Flow-DPPO and \(1 \times 10^{-6}\) for Flow-DPPO+CPS due to their different KL-scaling factors.
We applied the stragegy proposed in MixGRPO~\citep{mixgrpo} on all baselines and proposed methods for faster convergence and better performance.
Specifically, we mix ODE and SDE sampling and randomly select 3 steps out of first half of the denoising steps for SDE sampling.
The noise level for SDE sampling ($\eta$ in CPS sampling) is set to $0.8$.

For Diffusion-NFT, we follow the official implementation for SD3.5 for the rest of the hyperparameters, such as EMA schedule.

\section{Additional Experimental Results}
\label{sec:appendix_results}

\subsection{Additional Training Curves}
\label{sec:appendix_training_curves}

We provide the training curves on SD3.5 for the single-reward setting in Figure~\ref{fig:train_row_sd35_single} (the multi-reward setting is in Figure~\ref{fig:train_row_sd35_multi} in the main body).
We also provide the FLUX2-9B multi-reward training curves in Figure~\ref{fig:train_row_flux_multi} and FLUX.1-dev in Figure~\ref{fig:train_row_flux1dev}.

\begin{figure*}[t]
  \centering
  \includegraphics[width=\linewidth]{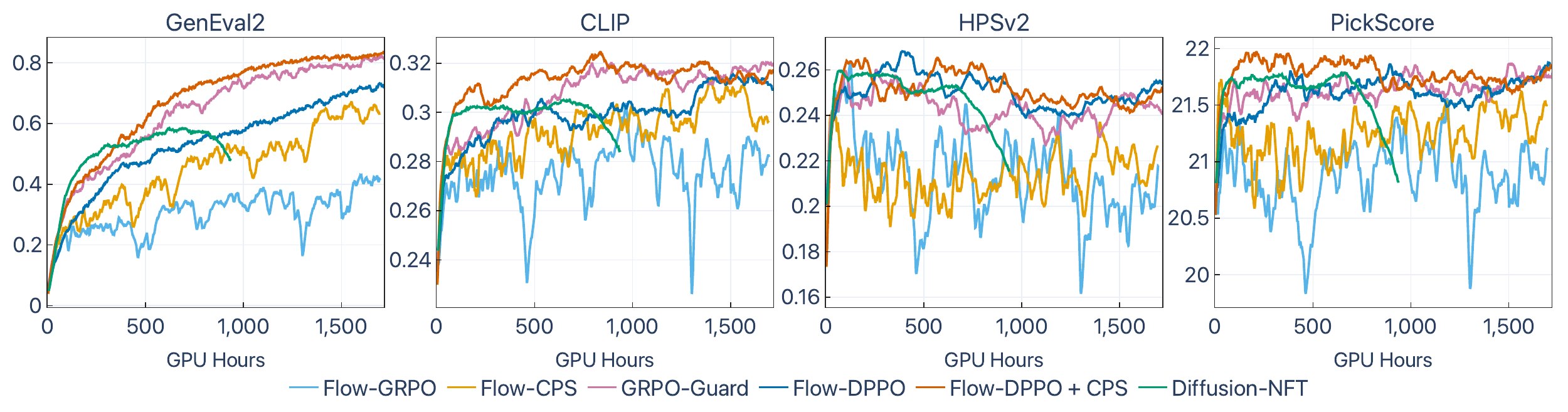}
  \caption{Training curves on SD3.5 for single-reward setting, including Diffusion-NFT~\citep{diffusionnft} as an additional baseline. Flow-DPPO variants achieve
state-of-the-art performance and less catastrophic forgetting on out-of-domain rewards, consistent with the main results.}
  \label{fig:train_row_sd35_single}
\end{figure*}

\begin{figure*}[t]
  \centering
  \includegraphics[width=\linewidth]{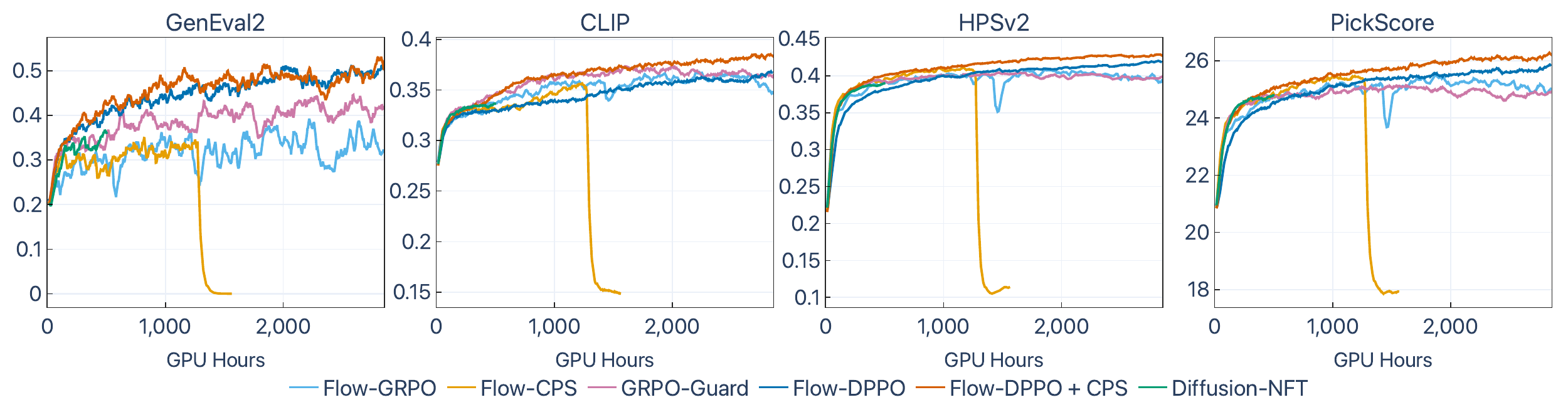}
  \caption{Training curves on FLUX2-9B for multi-reward setting (GPU hours). Flow-DPPO variants consistently outperform the baselines across all metrics, with a notable improvement on the GenEval2 reward.}
  \label{fig:train_row_flux_multi}
\end{figure*}

We additionally provide training curves on FLUX.1-dev in Figure~\ref{fig:train_row_flux1dev}.

\begin{figure*}[t]
  \centering
  \includegraphics[width=\linewidth]{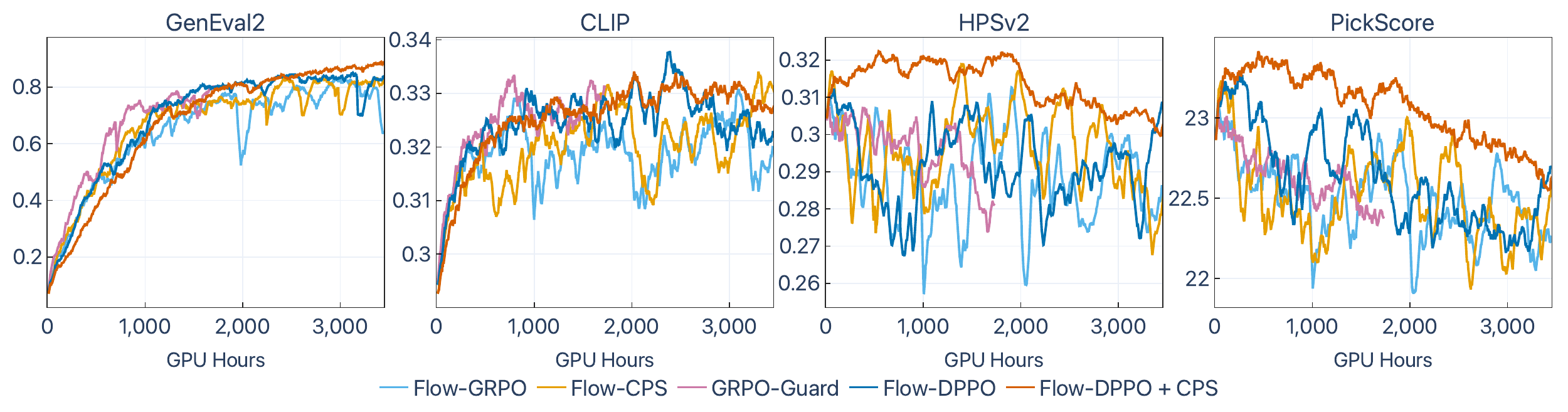}
  \caption{Training curves on FLUX.1-dev for single-reward setting.}
  \label{fig:train_row_flux1dev}
\end{figure*}

\begin{figure*}[t]
  \centering
  \includegraphics[width=\linewidth]{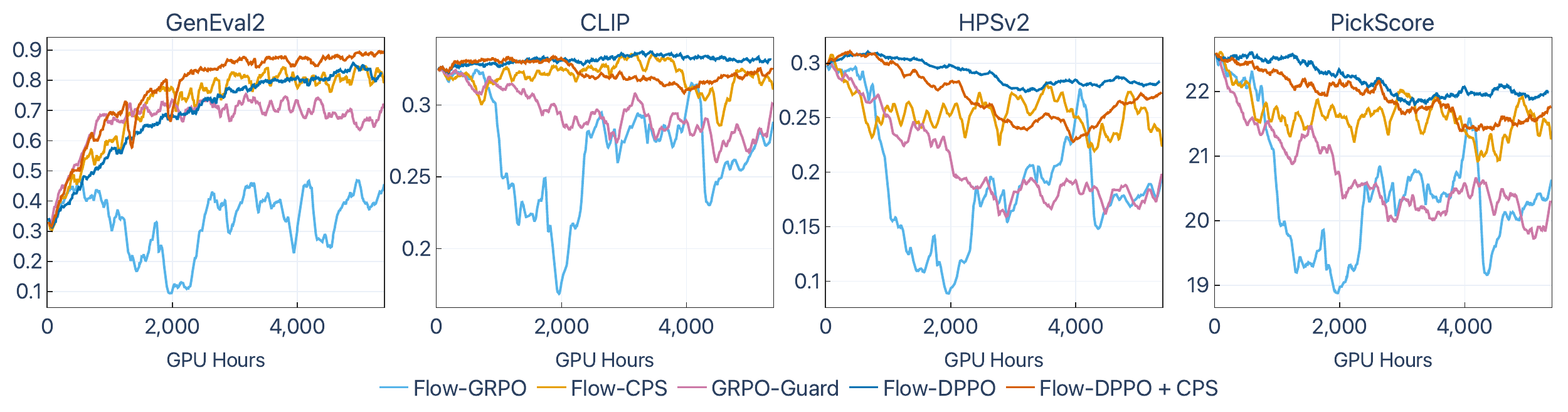}
  \caption{Training curves on FLUX2-9B with CFG scale 4.0. Flow-DPPO variants remain robust under CFG, achieving strong performance with less catastrophic forgetting.}
  \label{fig:train_row_flux_single}
\end{figure*}

\begin{figure*}[htbp]
  \centering
  \includegraphics[width=\linewidth]{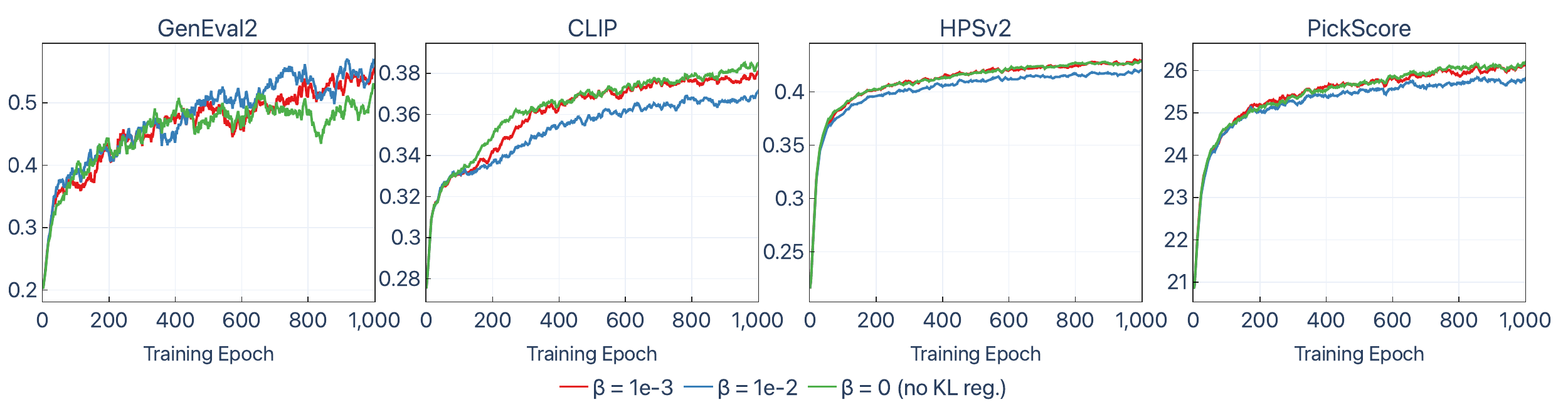}
  \caption{Training reward curves under three $\KL(\pi_\theta \| \pi_{\rm ref})$
  regularization strengths ($\beta$) on FLUX2-klein-base-9B (multi-reward GDPO, CPS schedule).
  A moderate $\beta{=}10^{-3}$ suppresses early reward hacking on PickScore and HPSv2,
  balancing cross-reward gradients and boosting final GenEval2 performance
  without hurting end-of-training performance on any individual reward.}
  \label{fig:kl_ref_beta_train}
\end{figure*}

\subsection{KL Divergence Curves}
\label{sec:appendix_kl_curves}

Figure~\ref{fig:kl_div} visualises the per-step KL divergence between
the current and reference (pre-trained) model across all six training settings
and two SDE schedules.
The corresponding end-of-training values are reported in
Table~\ref{tab:kl_final_step} of the main body.

\begin{figure*}[t]
  \centering
  \includegraphics[width=\linewidth]{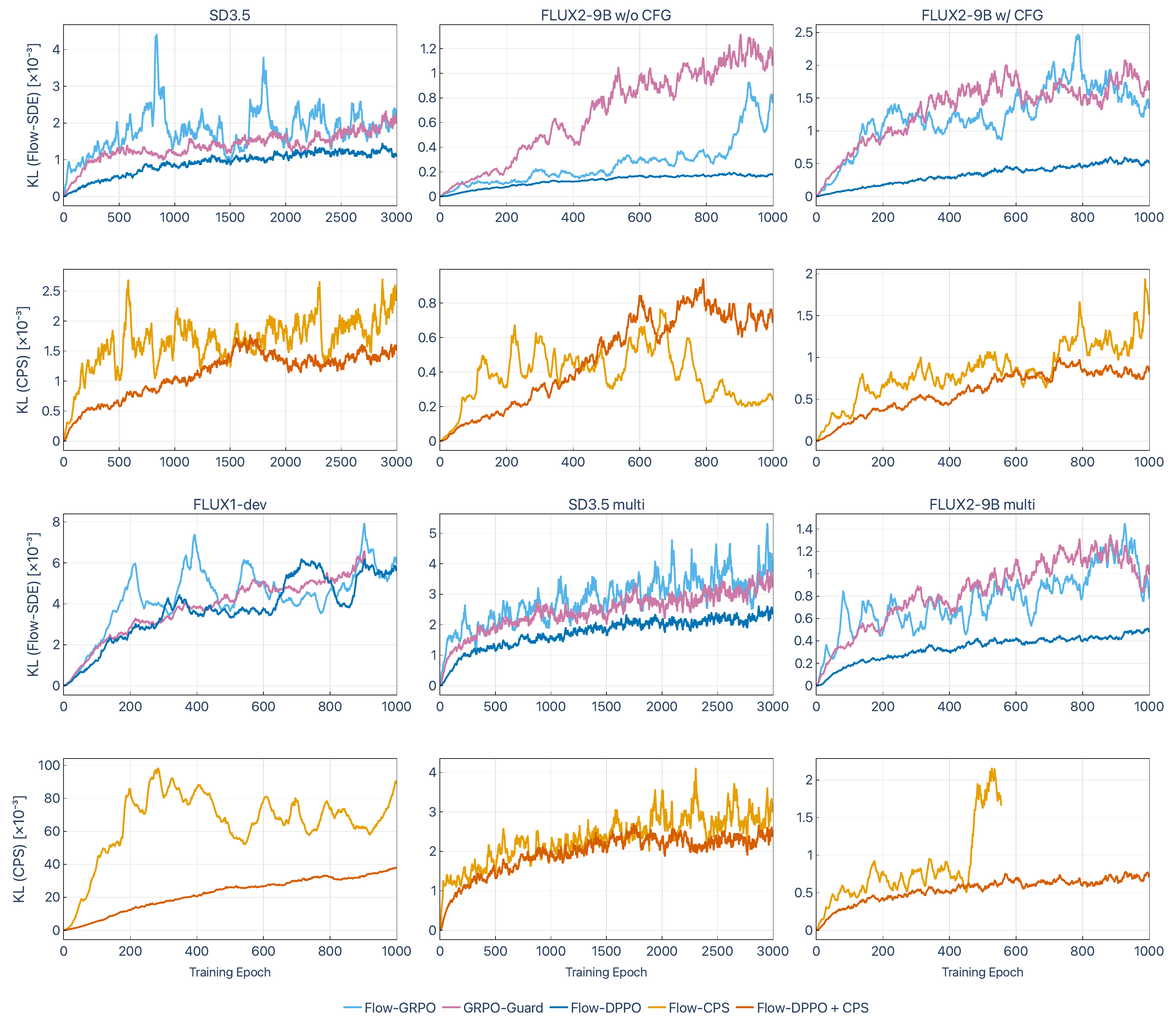}
  \caption{KL-divergence between the current and reference (pre-trained) model during training,
across six training settings (columns: four single-reward --- SD3.5, FLUX2-9B w/o CFG, FLUX2-9B w/ CFG, FLUX.1-dev;
two multi-reward --- SD3.5 multi, FLUX2-9B multi)
and two SDE schedules (rows: Flow-SDE, CPS).
For each schedule, Flow-DPPO variants maintain a lower KL divergence with the pre-trained model,
indicating less catastrophic forgetting and reward hacking.
The only exception is in the FLUX2-9B w/o CFG setting under the CPS schedule,
where Flow-DPPO + CPS shows a higher KL divergence than the Flow-CPS baseline after about epoch 500.
The Flow-CPS run on FLUX2-9B multi collapsed at epoch 480; we plot its
full logged trajectory and report its end-of-training KL at the run's last
logged step in Table~\ref{tab:kl_final_step}.}
\label{fig:kl_div}
\label{fig:kl_div_multi}
\end{figure*}

\subsection{Ablation Studies}
\label{sec:appendix_ablation_studies}

\subsubsection{Classifier-Free Guidance}
\label{sec:appendix_cfg_ablation}

Previous works found that CFG heavily affects the training convergence and performance~\citep{diffusionnft}.
Here, we study the effect of CFG on the training of Flow-DPPO on FLUX2-9B, as shown in Figure~\ref{fig:train_row_flux_single},
where the CFG scale is set to 4.0 following the official recommendation.
With CFG, Flow-DPPO variants still achieve state-of-the-art performance on the training reward (GenEval2)
and mitigate catastrophic forgetting on the out-of-domain prompts,
consistent with the observations in previous discussions.
This shows that the divergence-based mask is robust under CFG and continues to deliver strong performance.


\subsubsection{Reference KL Regularization Strength}
\label{sec:appendix_kl_ref_ablation}

We ablate the strength of the $\KL(\pi_\theta \| \pi_{\rm ref})$ regularization
term (controlled by $\beta$) on FLUX2-klein-base-9B under the multi-reward GDPO
setting with CPS scheduling.
Figure~\ref{fig:kl_ref_beta_train} shows the training reward curves and
Figure~\ref{fig:kl_ref_beta_kl} shows the KL divergence from the pretrained model.
A moderate regularization strength ($\beta{=}10^{-3}$) further mitigates
early-stage reward hacking on auxiliary objectives (PickScore, HPSv2, etc.),
thereby balancing the gradients across rewards and yielding an additional
improvement in final GenEval2 performance over the unregularized baseline,
without degrading end-of-training performance on any individual reward.


\begin{figure*}[t]
  \centering
  \includegraphics[width=0.75\linewidth]{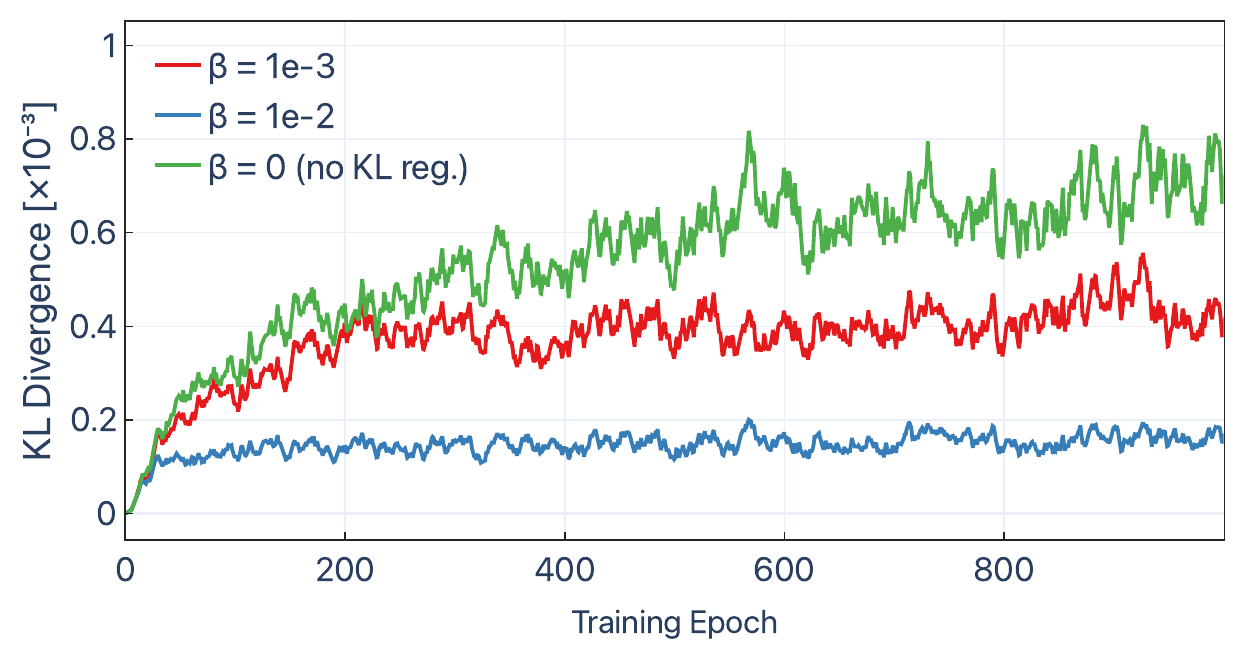}
  \caption{$\KL(\pi_\theta \| \pi_{\rm ref})$ during training for different $\beta$ settings.}
  \label{fig:kl_ref_beta_kl}
\end{figure*}

\definecolor{ourrow}{HTML}{E8F1FF}

\begin{table*}[t]
\centering
\caption{End-of-training Soft TIFA$_{\text{GM}}$ on GenEval2 (\%) across six training configurations (columns) and five RL algorithms (rows). The six columns correspond, left-to-right, to Figs.~\ref{fig:train_row_flux_single_nocfg}, \ref{fig:train_row_flux_multi}, \ref{fig:train_row_flux_single}, \ref{fig:train_row_sd35_single}, \ref{fig:train_row_sd35_multi}, \ref{fig:train_row_flux1dev}. Per-column \textbf{bold} and \underline{underline} mark the top-1 and top-2 methods; blue rows highlight our two contributions.}
\label{tab:appendix_geneval2_grid}
\begin{tabular}{lcccccc}
\toprule
& \multicolumn{3}{c}{\textbf{FLUX2-9B}} & \multicolumn{2}{c}{\textbf{SD3.5}} & \multicolumn{1}{c}{\textbf{FLUX.1-dev}} \\
\cmidrule(lr){2-4} \cmidrule(lr){5-6} \cmidrule(lr){7-7}
\textbf{Method} & Single & Multi & +CFG & Single & Multi & Single \\
\midrule
Flow-GRPO & 84.5 & 46.8 & 54.6 & 56.6 & 39.9 & 87.8 \\
Flow-CPS & 82.7 & 47.1 & \underline{89.0} & 74.8 & 44.6 & \underline{91.2} \\
GRPO-Guard & 82.8 & 49.0 & 78.8 & \textbf{85.8} & 47.8 & 87.6 \\
Diffusion-NFT & -- & 47.3 & -- & 64.5 & 42.5 & -- \\
\rowcolor{ourrow} Flow-DPPO & \underline{85.1} & \textbf{57.7} & 87.4 & 78.9 & \underline{48.1} & 90.7 \\
\rowcolor{ourrow} Flow-DPPO + CPS & \textbf{92.6} & \underline{55.2} & \textbf{91.0} & \underline{84.1} & \textbf{51.6} & \textbf{91.6} \\
\bottomrule
\end{tabular}
\end{table*}

\subsection{Quantitative Summary on GenEval2}
\label{sec:appendix_geneval2_quant}

To complement the per-setting training-curve figures above,
Tables~\ref{tab:appendix_geneval2_baselines} and \ref{tab:appendix_geneval2_grid} report the end-of-training Soft TIFA$_{\text{GM}}$ score on GenEval2 for each method.
Table~\ref{tab:appendix_geneval2_baselines} additionally reports end-of-training ancillary CLIP, PickScore, and HPSv2 rewards on both the in-domain GenEval2 prompt set and the held-out out-of-domain PickScore validation prompts,
contextualising both SD3.5-medium and FLUX2-klein-base-9B by stacking six blocks:
published reference numbers for state-of-the-art text-to-image systems,
the corresponding pretrained-baseline scores (no RL),
and the five RL fine-tuning algorithms applied to each base model under both the single-reward (GenEval2-only) and multi-reward (GenEval2~+~CLIP~+~PickScore~+~HPSv2) configurations.
Table~\ref{tab:appendix_geneval2_grid} then expands the per-method Soft TIFA$_{\text{GM}}$ comparison to all five training settings reported in this paper.

\definecolor{ourrow}{HTML}{E8F1FF}

\begin{table*}[t]
\centering
\setlength{\tabcolsep}{4pt}
\caption{GenEval2 [Soft TIFA$_{\text{GM}}$, defined in \citep{geneval2}] together with ancillary CLIP, PickScore, and HPSv2 rewards at the end of training. The four in-domain columns are evaluated on the GenEval2 prompt set (the official released evaluation set of 800 prompts); the three out-of-domain columns are evaluated on the PickScore prompt set. Within each RL block, \textbf{bold} marks the per-column top-1 method and \underline{underline} the per-column top-2 method. Blue rows highlight our two contributions.}
\label{tab:appendix_geneval2_baselines}
\resizebox{0.9\textwidth}{!}{%
\begin{tabular}{lccccccc}
\toprule
& \multicolumn{4}{c}{\textbf{In-Domain (GenEval2)}} & \multicolumn{3}{c}{\textbf{Out-of-Domain (PickScore)}} \\
\cmidrule(lr){2-5} \cmidrule(lr){6-8}
\textbf{Model} & GenEval2 & CLIP & PickScore & HPSv2 & CLIP & PickScore & HPSv2 \\
\midrule
\multicolumn{8}{l}{\textit{State-of-the-Art T2I Models}} \\
\quad SD3.5-large & 22.8 & -- & -- & -- & -- & -- & -- \\
\quad Bagel + CoT & 23.1 & -- & -- & -- & -- & -- & -- \\
\quad Qwen-Image & 33.8 & -- & -- & -- & -- & -- & -- \\
\quad Gemini 2.5 Flash Image & 44.6 & -- & -- & -- & -- & -- & -- \\
\midrule
\multicolumn{8}{l}{\textit{Pretrained baselines (before RL)}} \\
\quad SD3.5-medium & 12.4 & 0.250 & 21.00 & 0.213 & 0.244 & 19.99 & 0.210 \\
\quad FLUX2-klein-base-9B & 25.4 & 0.281 & 20.92 & 0.228 & 0.254 & 20.05 & 0.230 \\
\quad FLUX.1-dev & 23.3 & 0.297 & 23.26 & 0.315 & 0.276 & 21.91 & 0.304 \\
\midrule
\multicolumn{8}{l}{\textit{SD3.5-medium, single-reward RL fine-tuning}} \\
\quad Flow-GRPO & 56.6 & 0.297 & 21.21 & 0.219 & 0.252 & 19.33 & 0.206 \\
\quad Flow-CPS & 74.8 & 0.313 & 21.68 & 0.235 & 0.260 & 19.94 & 0.220 \\
\quad GRPO-Guard & \textbf{85.8} & \textbf{0.328} & \underline{22.03} & 0.252 & 0.265 & 19.94 & 0.214 \\
\quad Diffusion-NFT & 64.5 & 0.307 & 21.69 & 0.251 & 0.262 & 20.24 & 0.239 \\
\rowcolor{ourrow} \quad Flow-DPPO & 78.9 & \underline{0.319} & \textbf{22.06} & \textbf{0.263} & \underline{0.265} & \underline{20.45} & \textbf{0.253} \\
\rowcolor{ourrow} \quad Flow-DPPO + CPS & \underline{84.1} & 0.316 & 21.99 & \underline{0.262} & \textbf{0.272} & \textbf{20.50} & \underline{0.246} \\
\midrule
\multicolumn{8}{l}{\textit{SD3.5-medium, multi-reward RL fine-tuning}} \\
\quad Flow-GRPO & 39.9 & 0.358 & 25.09 & 0.399 & \underline{0.273} & 22.07 & 0.349 \\
\quad Flow-CPS & 44.6 & \underline{0.359} & 25.51 & 0.407 & 0.265 & 22.08 & 0.343 \\
\quad GRPO-Guard & 47.8 & 0.353 & \underline{25.64} & \underline{0.409} & 0.272 & 22.32 & 0.354 \\
\quad Diffusion-NFT & 42.5 & 0.334 & 25.30 & 0.394 & 0.269 & \underline{22.52} & 0.355 \\
\rowcolor{ourrow} \quad Flow-DPPO & \underline{48.1} & 0.345 & 25.63 & 0.409 & 0.273 & \textbf{22.58} & \underline{0.360} \\
\rowcolor{ourrow} \quad Flow-DPPO + CPS & \textbf{51.6} & \textbf{0.369} & \textbf{25.72} & \textbf{0.415} & \textbf{0.279} & 22.51 & \textbf{0.361} \\
\midrule
\multicolumn{8}{l}{\textit{FLUX2-klein-base-9B, single-reward RL fine-tuning}} \\
\quad Flow-GRPO & 84.5 & 0.314 & 21.82 & 0.276 & 0.264 & 20.84 & \underline{0.280} \\
\quad Flow-CPS & 82.7 & 0.311 & 21.82 & 0.261 & \underline{0.275} & \underline{21.15} & 0.267 \\
\quad GRPO-Guard & 82.8 & 0.312 & 20.52 & 0.210 & 0.230 & 18.45 & 0.167 \\
\rowcolor{ourrow} \quad Flow-DPPO & \underline{85.1} & \textbf{0.331} & \textbf{22.22} & \textbf{0.294} & \textbf{0.278} & \textbf{21.27} & \textbf{0.285} \\
\rowcolor{ourrow} \quad Flow-DPPO + CPS & \textbf{92.6} & \underline{0.315} & \underline{21.97} & \underline{0.279} & 0.265 & 20.79 & 0.272 \\
\midrule
\multicolumn{8}{l}{\textit{FLUX2-klein-base-9B, multi-reward RL fine-tuning}} \\
\quad Flow-GRPO & 46.8 & 0.371 & 25.61 & 0.412 & 0.277 & 22.62 & 0.357 \\
\quad Flow-CPS & 47.1 & 0.361 & 25.70 & 0.416 & 0.276 & 22.85 & 0.364 \\
\quad GRPO-Guard & 49.0 & \underline{0.375} & 25.27 & 0.411 & 0.269 & 21.99 & 0.349 \\
\quad Diffusion-NFT & 47.3 & 0.336 & 24.87 & 0.389 & 0.274 & 22.47 & 0.351 \\
\rowcolor{ourrow} \quad Flow-DPPO & \textbf{57.7} & 0.364 & \underline{25.76} & \underline{0.418} & \underline{0.282} & \underline{22.90} & \underline{0.368} \\
\rowcolor{ourrow} \quad Flow-DPPO + CPS & \underline{55.2} & \textbf{0.386} & \textbf{26.15} & \textbf{0.427} & \textbf{0.287} & \textbf{22.97} & \textbf{0.370} \\
\midrule
\multicolumn{8}{l}{\textit{FLUX.1-dev, single-reward RL fine-tuning}} \\
\quad Flow-GRPO & 87.8 & 0.331 & 23.03 & 0.311 & \textbf{0.291} & 21.85 & \textbf{0.311} \\
\quad Flow-CPS & \underline{91.2} & 0.328 & \underline{23.20} & 0.317 & 0.288 & \textbf{21.98} & \underline{0.307} \\
\quad GRPO-Guard & 87.6 & \textbf{0.333} & 22.69 & 0.293 & 0.286 & 21.03 & 0.276 \\
\rowcolor{ourrow} \quad Flow-DPPO & 90.7 & \underline{0.331} & 23.15 & \textbf{0.323} & \underline{0.290} & 21.60 & 0.300 \\
\rowcolor{ourrow} \quad Flow-DPPO + CPS & \textbf{91.6} & 0.331 & \textbf{23.29} & \underline{0.322} & 0.289 & \underline{21.91} & 0.305 \\
\bottomrule
\end{tabular}}
\end{table*}

\clearpage


\end{document}